\documentclass[11pt]{article}
\usepackage[letterpaper, margin=1in]{geometry}

\usepackage{adjustbox}
\usepackage{natbib}
\usepackage{graphicx}
\usepackage{subfigure}
\usepackage{caption}
\usepackage[utf8]{inputenc} % allow utf-8 input
\usepackage{amsmath}
\usepackage{amssymb}
\usepackage{mathtools}
\usepackage{amsthm}
\usepackage{txfonts}
\usepackage[lined,ruled]{algorithm2e}
\usepackage{algorithmic}
\usepackage{hyperref}

\usepackage{amsmath,amsfonts,bm}

\def\eqref#1{equation~\ref{#1}}
\def\1{\bm{1}}

\DeclareMathAlphabet{\mathsfit}{\encodingdefault}{\sfdefault}{m}{sl}
\SetMathAlphabet{\mathsfit}{bold}{\encodingdefault}{\sfdefault}{bx}{n}

\def\gE{{\mathcal{E}}}

\def\gH{{\mathcal{H}}}

\def\gV{{\mathcal{V}}}

\DeclareMathOperator*{\argmax}{arg\,max}
\DeclareMathOperator*{\argmin}{arg\,min}

\theoremstyle{plain}

\newtheorem{theorem}{Theorem}[section]
\newtheorem{example}{Example}[section]
\newtheorem{proposition}{Proposition}
\newtheorem{lemma}{Lemma}

\theoremstyle{definition}
\newtheorem{definition}{Definition}

\theoremstyle{remark}

\usepackage[T1]{fontenc}    % use 8-bit T1 fonts
\usepackage{hyperref}       % hyperlinks
\usepackage{url}            % simple URL typesetting
\usepackage{booktabs}       % professional-quality tables
\usepackage{amsfonts}       % blackboard math symbols
\usepackage{nicefrac}       % compact symbols for 1/2, etc.
\usepackage{microtype}
\usepackage{graphicx}
\usepackage{subfigure}
\usepackage{booktabs} % for professional tables

\usepackage{quiver}
\usepackage{hyperref}
\usepackage{lipsum}
\usepackage{fancyhdr}       % header
\usepackage{graphicx}       % graphics
\graphicspath{{media/}}     % organize your images and other figures under media/ folder
\usepackage{tikz}
\usetikzlibrary{fit,tikzmark}   
\usepackage{booktabs,cellspace}

\usepackage{xcolor}
\usepackage{colortbl}
\newcommand{\cfbox}[2]{%
    \colorlet{currentcolor}{.}%
    {\color{#1}%
    \fbox{\color{currentcolor}#2}}%
}
\newcommand{\red}[1]{\cfbox{red}{\color{red}\textbf{#1}}}
\newcommand{\gray}[1]{\cfbox{gray}{\color{gray}\textbf{#1}}}

\newcommand{\lightgrayarea}[1]{{\cellcolor{lightgray} \textbf{#1}}}

\title{Adversarial Label Invariant Graph Data Augmentations for Out-of-Distribution Generalization}

\author{Simon Zhang\thanks{Department of Computer Science, Purdue University, West Lafayette, USA} \qquad Ryan P. DeMilt\thanks{Department of Computer Science and Engineering, The Ohio State University, Columbus Ohio, USA} \qquad  Kun Jin\thanks{Department of Computer Science and Engineering, The Ohio State University, Columbus Ohio, USA} \qquad Cathy H. Xia\thanks{Department of Industrial and Systems Engineering, The Ohio State University, Columbus Ohio, USA}}
\date{}

\begin{document}
\maketitle

\begin{abstract}
Out-of-distribution (OoD) generalization occurs when representation learning encounters a distribution shift. This occurs frequently in practice when training and testing data come from different environments. Covariate shift is a type of distribution shift that occurs only in the input data, while the concept distribution stays invariant. 
We propose RIA - Regularization for Invariance with Adversarial training, a new method for OoD generalization under convariate shift. Motivated by an analogy to $Q$-learning, it performs an adversarial exploration for counterfactual data environments. 
These new environments are induced by \emph{adversarial label invariant data augmentations} that prevent a collapse to an in-distribution trained learner. It works with many existing OoD generalization methods for covariate shift that can be formulated as constrained optimization problems. 
%We devise a new training algorithm for OoD generalization that performs an adversarial search for training data environments that works with any existing OoD generalization method for covariate shift.
%or have assumptions about the data that aren't widespread in practice. 
  %We consider an adversarial data augmentation approach to obtain a minimax formulation of the OoD generalization problem. %This approach minimally changes the architecture.  %Training environments are often scarce and may only weakly satisfy the underlying data generation assumptions. Such cases are common for node classification datasets. %We call achieving extrapolation in such a case as OoD robustness. 
  %In this paper, in order to obtain OoD robustness under such conditions, we explore new adversarially data augmented environments with a policy gradient based method. 
  %We show a bound on the empirical adversarial error for graph classification. %adversarial data augmentation in the limit of an infinite number of graphs, we obtain a solution whose causal noise also lies near a high dimensional sphere. 
  We develop an alternating gradient descent-ascent algorithm to solve the problem in the context of causally generated graph data, and 
  perform extensive experiments %\footnote{\url{https://anonymous.4open.science/r/RIA-879B}} 
 on OoD graph classification for various kinds of synthetic and natural distribution shifts. We demonstrate that our method can achieve high accuracy compared with OoD baselines.
\end{abstract}

\section{Introduction}\label{sec: introduction}
The out-of-distribution (OoD) generalization problem is an important topic in machine learning \cite{li2022out, shen2021towards} where one attempts to extrapolate from training data to in-the-wild distribution shifted data. For example, in computer vision this is commonly demonstrated by the example of identifying cows vs. camels on green or sandy backgrounds \cite{beery2018recognition} or the colored MNIST example from \citep{arjovsky2019invariant}. Covariate shift is when the covariate, or input, distribution shifts while the concept distribution does not change. These varying data conditions are known as varying environments, which can be defined as data distributions conditioned on some varying environmental factors. A covariate shift is an example of a change in environment. 
%At training time we are given just a few, or possibly one, such environments. 
\textcolor{black}{Common approaches such as Empirical Risk Minimization (ERM), which selects a model with minimal loss over an average of the training environments, cannot generalize to OoD test data as} the training environment(s) often rarely reflect the testing environments. Thus OoD generalization requires specialized methods and assumptions beyond minimizing the loss over the training environment(s).

When there is covariate shift, \textcolor{black}{the distribution of input data shifts due to the change of environments}.
For various reasons, there may be a scarcity of training environments. It is common, in fact, to just have a few, or possibly one, training environment. Existing OoD generalization methods are based on the concept of achieving invariance, or stability amongst learners on various environments. Due to the lack of diverse training environments, there is a possibility of such a learner collapsing to an ERM solution.

Non-Euclidean data such as graphs offer new challenges to the problem of OoD generalization. The primary challenge is the variable structure of the graphs. The number of nodes of each graph is variable and the interconnection structure of a graph is represented by a $0$-$1$ matrix space different from the graph signal space of node attributes. It is particularly computationally expensive to handle the edges whose count grows quadratically in the number of nodes. %Furthermore, because of the interconnect structure's discrete nature, gradients can be difficult to compute if viewed as tensors in Euclidean space. 
Both tensors must be accounted for to define a graph. Furthermore, graphs have the permutation invariance inductive bias.%Thus at the data level, one must take into account the node index symmetries such as equivariance of data transformations. One must also acknowledge that the data is represented in very high dimensional space. %For node classification in particular, each node's label must be predicted individually using the interconnected structure of the graph. We can thus assume that the underlying data generation process for the graph may no longer involve node-wise independence and identically distributed behavior \cite{wu2022handling}. Furthermore, the environments available at training time may be limited. For example, the environments may depend on data collected over a specific large time interval and rare location with the data presented only as a few large graphs \cite{hu2020open, traud2012social}. 

%Furthermore, in \cite{wu2022handling} an adversarial approach to learning invariance for node classification is taken. These existing approaches lack node attribute level modifications to the graph and in particular lack being able to generate environments in the form of interventional distributions \cite{wu2022discovering} that do not come directly from the data.

%The SEM introduced in \cite{wang2022out} assumes a common concept distribution (conditional distribution of label given covariates) and shows that using certain kinds of data augmentations are enough to guarantee the existence of a solution to OoD generalization under covariate shift. This approach, however, does not model noise in the causal map and is unexplainable. To address this potential issue, we obtain a more robust OoD generalization by extrapolating with training environments from an adversarially trained distribution. %Where by robust generalization, we mean that scarce training data not entirely satisfying the data generation assumptions can be extrapolated from.

We will assume a common concept distribution across environments and only covariate shift exists between training and testing distributions. Existing OoD solution methods do not prevent the collapse to an ERM solution during training due to a lack of diverse training environments.
We design an algorithm to search, using alternating gradient descent-ascent, 
for counterfactually generated environments that are hard to learn. This adversarial search prevents collapse to an ERM solution by introducing difficult and diverse environments. %Our reward function is based on VRex \cite{krueger2021out, xie2020risk} and thus an approximation to MM-Rex. MM-Rex is known geometrically to extrapolate environments to outside the convex hull of the training environments. 
%The predictions vary across the explored environments under label preserving data augmentations. We show that as long as the adversarially augmented training environments satisfy our data generation assumptions, then we can still extrapolate to a solution of the OoD generalization problem. %Unlike in previous approaches the solution is obtained without changing the architecture of the GNN, without costly dense adjacency matrix gradients~\cite{wu2022handling}, or generating artificial environments that may go out-of-distribution~\cite{wu2022handling}.

%We present an approach based on causal invariance whereby attribute-level data augmentation is applied to the graph data. Conditioning on a mixture of these data augmentations with the training environments, we apply structural interventions at the representation level to learn an edge mask that stabilizes the representation across environments, similar to as in \cite{wu2022discovering}. 

The contributions of this paper are as follows:
\begin{enumerate}
%\item We identify a crucial component in the data generation process for the concept distribution under covariate shift: causal noise. We model causal noise by a noisy map in a structural equational model. We assume a generating map for all these noisy maps. The causal noise is quantified by the gap between that generating map and all other causal maps, called the causal gap.% for the concept conditional distribution for graphs. 
\item We formulate a causal data generation process for graphs. This model separates spurious and causal factors that determine the graph label.
\item We identify a common issue with many existing OoD solutions, namely when there is a ``collapse", or fitting, to the ERM solution. We briefly discuss this phenomenon in the context of our graph data model.  
\item We formulate what an adversarial label invariant data augmentation is and the counterfactual training distribution it can generate.
%\item We show bounds for the gap between the OoD risk, the ERM risk and any other risk from the training and show that adversarial data augmentations have sufficient power to solve the OoD generalization problem.
\item We introduce RIA: Regularization for Invariance with Adversarial training, a black-box defense to learn more environments for improved OoD generalization. The approach simulates counterfactual test environments in the form of a black-box evasion attack. This is motivated by an analogy to $Q$-learning.
%\item %We introduce graph adversarial data augmentation to resolve causal noise%, represented by a noisy causal map in the SEM
%We show under sufficient smoothness assumptions and using adversarial data augmentation that an OoD solution under covariate shift due to the stochasticity from the causal noise distribution can be approximated uniformly at a rate of $O(\frac{1}{\sqrt{n}})$ in the limit of infinite samples to % with probability $1-\theta$ for $\theta>0$ to $O(r+\frac{log(\frac{1}{\theta})}{\sqrt{n}})$, for $n$  the amount of data samples and 
%$O(r)$ where $r$ bounds the causal gap. 
    
%\item We perform extensive experiments demonstrating the wide ranging effectiveness of our approach against recent state of the art node classification OoD generalization approaches, including RICE \cite{wang2022out} and EERM \cite{wu2022handling}.
\item We perform extensive experiments to demonstrate the effective OoD generalizability of our method on real world as well as synthetic datasets by comparing with existing %and representative 
graph OoD generalization approaches. %including RICE~\cite{wang2022out} and EERM~\cite{wu2022handling} as well as in-distribution datasets. %We obtain up to 2.7\% accuracy improvement on ogb-arxiv and elliptic in the inductive setting 
\end{enumerate}

\section{Related Work} A common approach to tackling the OoD problem is to find a representation that performs stably across multiple environments~\cite{arjovsky2019invariant, bagnell2005robust,ben2009robust,  chang2020invariant,duchi2016statistics, krueger2021out, liu2021heterogeneous, mahajan2021domain, mitrovic2020representation,sinha2017certifying}. The goal of such an approach is to eliminate spurious or shortcut correlations that would normally be learned through empirical risk minimization (ERM). ERM is the common approach taken in machine learning to minimize the training error over a union of training environments in order to achieve well known generalization bounds \cite{NIPS1991_ff4d5fbb}. For graph data, \cite{wu2022discovering}  assume an underlying data generation process, then their assumptions provide a guarantee \cite{xie2020risk} that they can learn a representation that is stable across environments. In their data generation assumptions, they assume graph data can be decomposed into causal and spurious parts. By learning stably across environments, their objective is to learn to ignore the spurious parts of the data.  

Adversarial training~\cite{croce2020robustbench,szegedy2013intriguing, goodfellow2014explaining,barreno2006can,kearns1988learning} is when a model is trained with adversarial examples. Adversarial examples~\cite{goodfellow2014explaining} are perturbations of the original data which change the output of a learner. 
When the adversarial examples are used to fool the learner~\cite{goodfellow2014explaining,moosavi2016deepfool,carlini2017towards}, this is called an adversarial attack. When the attack is on the testing data, this is called an evasion attack~\cite{biggio2013evasion}. Adversarial training is a defense to these kinds of attacks. 

Non-Euclidean data such as graphs offer new challenges to the OoD problem. Many of the existing works on this topic are explained in the survey \cite{li2022out}.

\section{Causal Data Generation Process}
\label{sec: causalmodel}
It is common for data to be generated through causality, or cause and effect relationships. We define structural causal models (SCM), which model these causal relationships in the data distribution. Underlying any SCM is a combinatorial object called a directed acyclic graph (DAG), whose edges can be used to model cause and effect. 
\begin{definition}
    A Directed Acyclic Graph (DAG) is a directed graph $G=(V,E), E \subseteq V \times V$ for $V= [n]= \{1,...,n\}$ where any directed path of nodes $(v_1,...,v_k)$ with $(v_i,v_{i+1}) \in E$ for $i=1,...,k-1$ cannot have $v_1=v_k$
\end{definition}
Consider a joint distribution $P(V_1,...,V_n)$ over random variables $\gV= \{V_i\}_{i=1}^n$. A random variable $V_i$ is observable if it can be sampled from $P(V_1,...,V_n)$ and hidden if it cannot be.  
\begin{definition}
Given a DAG $G=(V=[n],E)$, define a structural causal model (SCM) $\mathcal{M}$ on $G$ as the following tuple: $(\mathcal{V}, \mathcal{F}, \mathcal{U})$ where $[n]$ indexes $\mathcal{V}$ and $\mathcal{U}$, meaning we can index every $V \in \mathcal{V}$ as $V=V_i$ for some $i \in [n]$ where $V_i\neq V_j$ if $i\neq j$ and similarly for $\mathcal{U}$.
The set $\mathcal{V}$ is a set of endogenous random variables. The set $\mathcal{U}$ is a set of exogenous random variables, each being i.i.d. uniform random variable in $[0,1]$. Each endogenous variable $V_i$ has a set of parents $V_{pa_i} \triangleq \{ V_j : (j,i) \in E\}$. If $pa_i$ is nonempty, we have the relationship:
\begin{equation}
    V_i = f_i(V_{pa_i}, U_i)
\end{equation}
where $f_i \in \mathcal{F}$ and $U_i \in \mathcal{U}$.

If $U_1 \perp U_2 \perp ... \perp U_n$ (joint independence), then the SCM is called Markovian. 
\end{definition}

For a Markovian SCM the joint distribution can be factored into conditional distributions for each endogenous variable~\cite{pearl2009causal}:
\begin{equation}
    P(V_1,...,V_n)= \Pi_{i=1}^n P(V_i \mid V_{pa_i}),
\end{equation}
where
$P(V_i \mid  V_{pa_i})
=P(V_i)$ if $V_{pa_i}=\emptyset$. 
%\textcolor{red}{(Why is Markovian needed? How do you know your graph is Markovian? Again, this seems stylized)}

\iffalse
In the formulation of our approach, there is no dependency on a particular OoD loss besides that it can be decomposed into the form $ERM(\mathcal{E}_{tr})+\textbf{OoD-Reg}_{\bullet}(\mathcal{E}_{tr})$ on training environment set $\mathcal{E}_{tr}$. Many existing methods require causal assumptions on the data generation process.  %It is often the case, however, that the causal assumptions are satisfied but the environments are not expressive enough for an environment invariant predictor to extrapolate OoD. %Within this regime or assuming a distribution shift induced by RIA where causal assumptions may become true,
\fi
%For graph data, we discuss a particular causal model that we use in subsequent guarantees. It is based on the causal model presented in \cite{arjovsky2019invariant}. RIA can address when these causal assumptions fail by adversarially augmenting the data so that there are diverse counterfactually generated OoD environments. To actually extrapolate to OoD data \emph{as intended} by existing methods, RIA would have to learn a distribution shift that can push the data towards these causal assumptions. Empirically, we find that preventing ERM collapse allows for OoD generalization even without a guarantee that the intended causal assumptions are satisfied.
\begin{figure}[t]
\centering
\includegraphics[width=0.8\columnwidth]{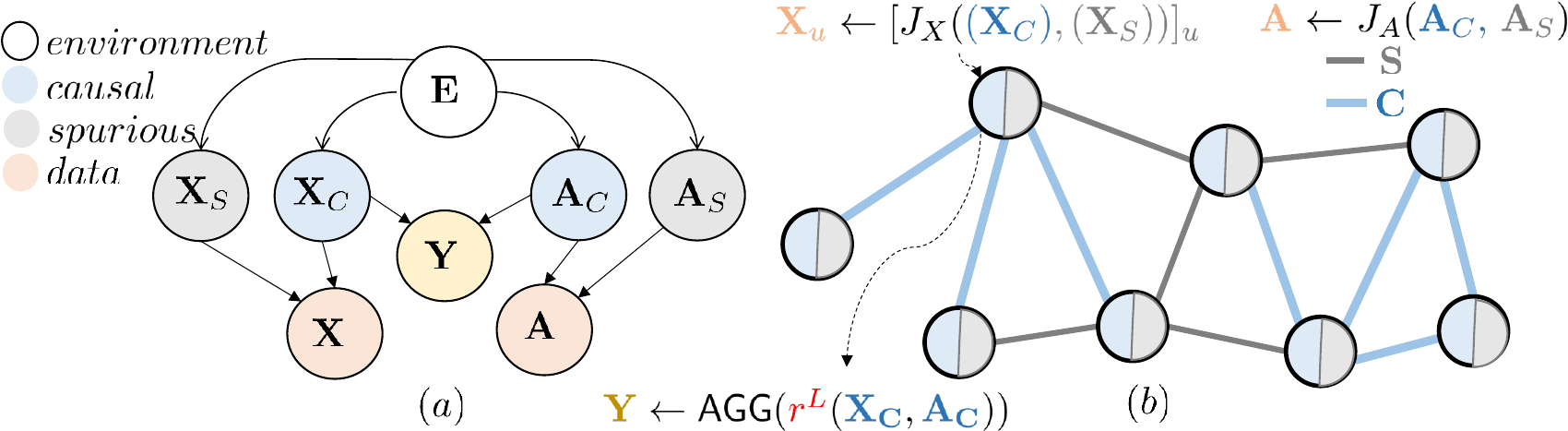}
\caption{(a): A casual graph for the data generation process. The exogenous variable $\textbf{E}$ is an integer that indexes a data environment. (b) A labeled attributed graph instance with the joining operation for causal/spurious attributes and edges shown. In the figure, the joining operation $J_X$ is shown as the node-wise concatenation of causal and spurious attribute tensors. The joining operation $J_A$ shown in the figure provides the sum of the adjacency matrices $\mathbf{A}_C$ and $\mathbf{A}_S$ where the hadamard product $\mathbf{A}_C \odot \mathbf{A}_S=\mathbf{0}$. The half grey color on nodes represents the $\mathbf{X_S}$ while the half blue color represents the $\mathbf{X}_C$. }
		\label{fig: SEM}
		\centering
	\end{figure}
\subsection{A Structural Causal Model for Graphs}
We will be using a specific data generation process to model the graph data distribution. It is based on the generic causal model presented in \cite{arjovsky2019invariant}. We define the following random variables: $\gV= \{\mathbf{E}, \mathbf{X}_C, \mathbf{X}_S, \mathbf{A}_C, \mathbf{A}_S,\mathbf{X},\mathbf{A},\mathbf{Y} \}$. 

The causal relationships between these variables are shown graphically by the directed edges in the DAG of Figure \ref{fig: SEM}.

The variable $\mathbf{E}$ is the exogenous environmental variable. It takes values from a finite set $\gE_{all}$. The environments determine various graph parameters. These can include:
\begin{enumerate}
    \item Having certain causal OR spurious properties of the graph topology such as e.g. treewidth, forbidden graph minors, isomorphism classes, spectral distributions etc. 
    \item AND Having certain causal OR spurious properties on the signal at the nodes: e.g. inherent embedding dimension, large magnitude moments, long tails, fat tails, pairwise correlation etc. 
\end{enumerate}
To generate a graph, it is necessary to have two tensor representations: a node attribute tensor and an adjacency matrix. The two tensor representations: $\mathbf{X}_C,\mathbf{A}_C$ are causal and are for the same number of nodes $n_e$. The two tensor representations: $\mathbf{X}_S,\mathbf{A}_S$ are spurious and are for the same number of nodes $n_e$. These two graphs are ``attached" through the tensor representations in the causal model. The attachment process is determined by the two deterministic concatenation maps $J_X, J_A$, see Figure \ref{fig: SEM}.
\iffalse
\begin{enumerate}
    \item The node attribute tensor is defined as follows:
\begin{equation}
    \mathbf{X}= J_X(\mathbf{X}_C, \mathbf{X_S})
\end{equation}
where $J_X$ is a deterministic column-wise concatenation map. 

\item The node attribute tensor is defined as follows:
\begin{equation}
    \mathbf{A}= J_A(\mathbf{A}_C, \mathbf{A}_S)
\end{equation}
where $J_A$ is a deterministic addition map. We assume that there is no agreement between $\mathbf{A}_C$ and $\mathbf{A}_S$ on the nonzeros.
\end{enumerate}
\fi

This graph generation process can generate any graph with less than or equal to $N_e$ number of nodes. The nodes that are not generated are given a ``NULL" node attribute. 

The ground truth label $\mathbf{Y}$ is generated by the following deterministic composition, see ~\cite{hamilton2017inductive}: 
\begin{equation}
    \mathbf{Y}=\textsf{AGG}(r^L(\mathbf{X}_C,\mathbf{A}_C))
\end{equation}
\begin{enumerate}
    \item The map  $r^L$ takes the pair of causal tensors and outputs a tensor of node representations. This tensor is indexed row-wise by the $N_e$ nodes.  At each node there is the composition of an $L$-hop local neighborhood recursive expansion map over a deterministic map $m: \mathbb{R}^d \rightarrow \mathbb{R}^d$: 
\begin{subequations}
\begin{equation}
    r^L(\mathbf{X}_C,\mathbf{A}_C)[v]\gets s^L(v)
    \end{equation}
\begin{equation}
s^L(v)\gets m(\mathbf{X}_C[v]+\sum_{ u \in \text{Nbd}(v)}s^{L-1}(u))
\end{equation}
\begin{equation}
    s^0(v)\gets m(\mathbf{X}_C[v])
\end{equation}
\end{subequations}
where $\diamond[\bullet]$ denotes the $\bullet$-th node row vector of tensor $\diamond$.
\item The  map $\textsf{AGG}: 2^{\mathbb{R}^D} \rightarrow \{0,1\}$ is a row-wise set map to the booleans $\{0,1\}$ over the tensor $r^L(\mathbf{X}_C,\mathbf{A}_C)$.
\end{enumerate}
%The graph $\mathbf{C}$ is called a causal graph. 
%That is, the causal variable $\mathbf{X}_C$ satisfies $\mathbf{Y}^e \Perp_{\mathbf{C}^e} \mathbf{X}^e$ and determines the label $\mathbf{Y}^e$ by the deterministic map $y$ up to some randomness $\eta \sim U[0,1]$. 

%The map $y$ is a deterministic map, it must determine the concept distribution $P(\mathbf{Y}\mid \mathbf{X}_C,\mathbf{A}_C)$~\cite{kallenberg1997foundations} (Proposition 8.20).

\textbf{Generating Graph Data: }
The data generation process proceeds from the exogenous environment variable through the chain of children over the SCM. The causal chains end on the observable covariate and label variables. All other variables are hidden.
\begin{enumerate}
    \item From the environmental variable $\mathbf{E}$ taking on environment $e \in \mathcal{E}_{all}$, two conditionally independent causal and a spurious graphs are randomly generated. 
    \item These two graphs are ``attached" to form the covariate data. For environment $e$, we denote the tensor representation as: $\mathbf{G}^e:=(\mathbf{X}^e,\mathbf{A}^e)$. 
    \item The causal graph is passed through a deterministic recursive neighborhood expansion map. For environment $e$, this produces a label $\mathbf{Y}^e$. 
    \item The covariate data and the label are paired to form the observable data: $(\mathbf{G}^e,\mathbf{Y}^e)$. We denote this distribution by $P^e$
\end{enumerate}
\section{The Out-of-Distribution Generalization Problem}
We will assume that there are in total only a finite number of environments. 
We also assume that there is a shift in the covariate distribution for testing different from the training distribution. The out-of-distribution generalization problem seeks to predict a label on any unseen testing distribution. Since we do not know the testing distribution(s), we optimize for worst case data distributions in the following minimax optimization problem, called the OoD generalization problem.  %over hypothesis class $\tilde{\mathfrak{H}}$ and %called the out-of-distribution generalization problem:
\begin{equation}\label{eq: SEMoptproblem}
    \textsf{OoD}(\mathcal{E}_{all})\triangleq\min_{{h} \in \gH}%\tilde{h}\in \tilde{\mathfrak{H}}} %\sup_{Q: W_{\infty}(Q_{(X,A)},P^{tr}_{(X,A)})<D_{\mathbf{\xi}}, (Q,P^{tr}_{(X,A)} \in \mathbb{P}_{r_1,r_2})}
    %\sup_{F: W_{\infty}(F,E)<D_{E}}
    %\sup_{\mathcal{A}} %supp(\mathcal{A})=supp(\Xi)}
    \sup_{e \in \mathcal{E}_{all}} 
    %\sup_{Q*F: W_{\infty}(Q_{(X,A)}*F,P^{}_{(X,A)}) )<D, (Q \in \mathbb{P}_{r_1,r_2},P \in \mathbb{P}_{r_1,r_2})} 
    R^{e}(h)
    %\mathbb{E}_{(\mathbf{G}^e,\mathbf{Y}^e) \sim P^e}[l_e({h}(\mathbf{G}^e),\mathbf{Y}^e)]
\end{equation}
where $\gH$ is a hypothesis space of boolean functions over graphs called learners. %In particular, $\gH$ adopts the addition and scalar multiplication operations from $\mathbb{R}^D$,
Let the risk of a learner $h \in \gH$ over an environment $e$ be defined as:
\begin{equation}
    R^e(h)\triangleq \mathbb{E}_{(\mathbf{G}^e,\mathbf{Y}^e) \sim P^e}[l({h}(\mathbf{G}^e),\mathbf{Y}^e)]
\end{equation}
%where $e$ indexes a specific environment,
The distribution $P^e$ is caused by the environment $e$. It is over the data $(\mathbf{G}^e,\mathbf{Y}^e)$, and {$h(\cdot)$ is a learner to predict ground truth target label $\mathbf{Y}^e$ from $\mathbf{G}^e$.
\begin{definition}  \label{def: environment}
Denote $\mathcal{E}_{all}$ the set of all environment indices that index all data distributions for some classification task that we want to learn. 
Let $\mathcal{E}_{tr} \subsetneq \mathcal{E}_{all}$ be a strict subset of training environments that are accessible during training. 
%\simon{and in particular, representable on a computer.}
\end{definition}
$\textbf{ERM: }$ 
When there is no distribution shift at all, the standard approach would be to take $\mathcal{E}_{tr}$, and minimize the average risk over these training environments. This is known as Empirical Risk Minimization (ERM), which is given in the following equation:
\begin{equation}\label{eq: ERMoptproblem}
    \textsf{ERM}(\mathcal{E}_{tr})\triangleq \min_{h}%\tilde{h} \in \tilde{\mathfrak{H}}} %\sup_{Q: W_{\infty}(Q_{(X,A)},P^{tr}_{(X,A)})<D_{\mathbf{\xi}}, (Q,P^{tr}_{(X,A)} \in \mathbb{P}_{r_1,r_2})}
    %\sup_{F: W_{\infty}(F,E)<D_{E}}
    %\sup_{Q \in \mathbb{P}_{r_1,r_2}} 
    %\sup_{Q*F: W_{\infty}(Q_{(X,A)}*F,P^{}_{(X,A)}) )<D, (Q \in \mathbb{P}_{r_1,r_2},P \in \mathbb{P}_{r_1,r_2})} 
     \frac{1}{\lvert \mathcal{E}_{tr}\rvert }\sum_{e \in \mathcal{E}_{tr}}R^e(h)%\mathbb{E}[l_e({h}(\mathbf{G}^e),\mathbf{Y}^e)]
\end{equation}
Let $h_{ERM}$ denote the minimizer to the ERM equation (e.g. zero risk). Standard generalization bounds for in-distribution testing data are known for ERM~\cite{vapnik1991principles}. However, these generalization bounds are invalid when there is a distribution shift of $P^e$ from training environments with $e \in \mathcal{E}_{tr}$ to testing distributions with  $e \in \mathcal{E}_{all}$~\cite{ahuja2021invariance}.

\textbf{IRM: }(~\cite{arjovsky2019invariant})  This is a bi-level optimization problem that learns 1. a single data embedding and 2. a downstream boolean predictor that minimizes jointly across all environments.
\begin{subequations}
\begin{equation}
    \min_{\Phi:  \mathcal{G} \rightarrow V,  w: V \rightarrow \{0,1\}} \sum_{e \in \mathcal{E}_{tr}} R^e(w \circ \Phi)
\end{equation}
\begin{equation}
    \text{ s.t. }w \in \argmin_{\bar{w}: V \rightarrow \{0,1\}} R^e(\bar{w}\circ \Phi), \forall e \in \mathcal{E}_{tr}
\end{equation}
\end{subequations}
\subsection{ERM Collapse}
When training over the training environments a common phenonenom called ERM collapse may occur, namely that the learner $h^*$ determined by a learning algorithm $\mathcal{A}: \Pi_{e \in \mathcal{E}_{tr}}{D}_{e} \rightarrow \mathcal{H}$ over the data sample sets of size $m$, $D_e \sim (P^e)^m: e \in \mathcal{E}_{tr}$ converges to the ERM solution: $h_{ERM} \in \mathcal{H}$.

In the context of out-of-distribution generalization and a learning algorithm that attempts to minimize each environmental risk, this can occur for \textbf{some} of the following reasons:
\begin{enumerate}
    %\item (No Disagreeing Samples) For any pair of environments $e,e' \in \mathcal{E}_{tr}$ and data sample pairs $(G^e,Y^e) \sim P^e,(G^{e'},Y^{e'}) \sim P^{e'}$: if  $G^e\cong G^{e'}$ (Isomorphic Graphs), then $Y^e=Y^{e'}$. 
    \item (Few Samples) There are very few training samples, none repeating, resulting in overfitting.
    \begin{enumerate}
        \item e.g. Learning on a single data sample%~\cite{zhang2024expressive}.
        \item e.g. A single data sample from one of three separate environments with common support.
    \end{enumerate}
     \item (Single Environment) There is only one training environment, making $h_{ERM}$ a feasible solution to converge to.
    \item (Zero Risk) The risk over all of $\mathcal{E}_{tr}$ is zero, making $h_{ERM}$ a feasible solution.
\end{enumerate}\label{prop: ERM-collapse-sufficient}
These properties are related in the following way:
\begin{proposition}(Properties of Sufficient Conditions for ERM collapse)

Assume the causal model of Figure \ref{fig: SEM} and assume $h_{ERM} \in \mathcal{H}$.

When all distributions $P^e, e \in \mathcal{E}_{tr}$ have common support:
\begin{enumerate}
    \item Case 1 (Few Samples)  implies Case 2 (Single Environment).
    \item Case 2 (Single Environment) implies Case 3 (Zero Risk).
\end{enumerate}
\end{proposition}
\begin{proof}
    \textbf{1. } When there are few samples:
    \begin{equation}
        D:=\bigcup_{e \in \gE_{tr}}D_e: D_e=\{s_e: s_e \sim P^{e}\}, \lvert D_e\rvert \ll \infty, 
    \end{equation} 
    We claim that $D$ forms an environment of its own. 
    
    The assignment $f: G\mapsto Y$ over all pairs $(G,Y) \in D$ in the data samples is well defined. Thus, there is a single consistent decision boundary over the samples. Thus, we can form the environment $\tilde{e}$ for the uniform distribution of all samples $(G,Y)$ over all the training environments. The new environment behaves as follows in the causal model:
    \begin{equation}
       \begin{tikzcd}
	& {E=\tilde{e}} & \\
	\emptyset & {\mathbf{A}} & {\mathbf{X}} \\
	{\mathbf{A}} & {\mathbf{X}} & {\mathbf{Y}}
	\arrow[from=1-2, to=2-1]
	\arrow[from=1-2, to=2-2]
	\arrow[from=1-2, to=2-3]
	\arrow["{\textsf{id}}"', from=2-2, to=3-1]
	\arrow[from=2-2, to=3-3]
	\arrow["{\textsf{id}}"'{pos=0.8}, from=2-3, to=3-2]
	\arrow[from=2-3, to=3-3]
\end{tikzcd}
    \end{equation}
The spurious graph tensors do not show up in environment $\tilde{e}$ and the function $f$ is represented by the relationship between the collider $\mathbf{Y}$ and its two parents.    

    \textbf{2. } If there is only a single environment, then there is no competing environment to prevent zero risk. Thus, since $h_{ERM} \in \mathcal{H}$, risk minimization over this only environment must result in zero risk.  
\end{proof}
\subsubsection{A Simple Example for Graphs}
In the context of our SCM graph data generation process, we give a very simple example of ERM collapse for the IRM learning algorithm:
\begin{example}
    Consider the following two environments:
    \begin{enumerate}
        \item $(e=1)$ A complete graph which has a decomposition into a causal  spanning tree with  its remaining spurious edges. All nodes have a constant signal $0$.
        \item $(e=2)$ A graph consisting of both causal and spurious undirected paths  of even number of nodes with signal $1$ at all nodes.
    \end{enumerate}
    Let $m: \mathbb{R} \rightarrow \mathbb{R}$ be the map $m(x):=x-1$ and let $L=1$.
    
    The ground truth label is predicted as for either environment:
    \begin{equation}
        \mathbf{Y}= \mathbf{1}_{\text{odd}}[\max_{v \in V(\mathbf{G}^e)}(\{\textsf{deg}(v): v \in V(\mathbf{G}^e)\})]
    \end{equation}
    which checks the parity of the maximum degree node and outputs 1 when the maximum degree of a node in $\mathbf{G}^e$ is odd.
    \begin{itemize}
        \item IRM with $w=1$ (IRMv1)~\cite{arjovsky2019invariant} will learn: $\Phi^*(G)={0}; G \sim \mathbf{G}^e$.
    \end{itemize}
    This achives zero risk for both environments, thus by Proposition \ref{prop: ERM-collapse-sufficient} we have ERM collapse. 

    This solution happens to not be the ground truth predictor, which would recognize that the spanning tree in environment one can have odd degree nodes.
\end{example}
\subsection{Adversarial Label Invariant Data Augmentations}
%In Section \ref{sec: motivation} we motivated finding a solution to the OoD generalization by data augmentation for a general case with the strong assumption of equality between minimax and maximin. In Section \ref{sec: ood-by-advaug} we loosened these assumptions. Corollary \ref{cor: OoD-by-advaug} says that the risk of the hardest environment for any learner $h$ can be simulated by some adversarial data augmentation over the training data we consider a method that augments any existing OoD generalization method.
We design a training algorithm for OoD generalization that adversarially explores data points by data augmentation for extrapolation beyond the training environments for OoD generalization. We focus on graph data, however our method can be generalized to any kind of data. The exploration is done by stochastic gradient ascent updates, adversarially maximizing against the ERM loss of any regularized OoD loss to search over environments \cite{yi2021improved}. The updates alternately minimizes the learner $h$ and data augmentations $\mathbf{a}$ for the $h$. 

In order to not violate the causality of our data generation process, the augmentations should \emph{not affect the map from causal graph to label}, see Figure \ref{fig: SEM}. The covariate graph data and the label share the causal graph variable as their common confounder. If an intervention on the covariates changes the ground truth label, then the learner would not know since the causal graph variable is hidden. Thus, we restrict our data augmentations to not change the label. Such data augmentations are called label invariant data augmentations:
\begin{definition}(Label Invariant Data Augmentation)

For covariate distribution $P$ and ground truth labeling function $f$, a label invariant data augmentation for $h$ is the following map:
\begin{equation}
    a: \textsf{supp}(P)\rightarrow \textsf{supp}(P) \text{ s.t. } f(a(X))=f(X)
    \end{equation}
\end{definition}
%A label invariant data augmentation is the most basic kind of data augmentation. 
A label invariant data augmentation only affects the ground truth label. In the data generation setting of~\cite{wang2022out}, it can be shown that causally invariant transformations are label invariant. Their setting requires a collapsed posterior.

In the case of our data generation process for graphs, data augmentations that only affect the spurious subgraph of an input graph $\mathbf{G}$ cannot change the ground truth label function. Thus such data augmentations are label invariant.  %and subsumes two related data augmentations: counterfactually invariant data augmentations and adversarial data augmentations. Both of these data augmentations involve the learner instead of the ground truth. 

A related data augmentation involves changing the output of the learner.
\iffalse
A counterfactually invariant data augmentation involves changing the environment: ``counter to the fact" through a data augmentation. This does not change the label of the learner:
\begin{definition}
(Counterfactually invariant data augmentation)

 Let $h$ be a learner and $e,e': e\neq e'$ be two environments where $\textsf{supp}(P^e)=\textsf{supp}(P^{e'})$. The following map is a counterfactually invariant data augmentation: 
\begin{equation}
    a: \textsf{supp}(P^e)\rightarrow \textsf{supp}(P^e) \text{ s.t. }
\end{equation}
\begin{equation}
    (P^e)_{\#}(a)=P^{e'} \text{ and }h(a(X))=h(X)
\end{equation}
\end{definition}
\fi
These are called adversarial data augmentations.
\begin{definition}
(Adversarial data augmentation) \cite{goodfellow2014explaining}

 Let $h$ be a learner and covariate distribution $P$,
\begin{equation}
    a: \textsf{supp}(P)\rightarrow \textsf{supp}(P) \text{ s.t. } h(a(X))\neq h(X)
\end{equation}
\end{definition}
We say a data augmentation is an \textbf{adversarial  label invariant data augmentation} if it is an adversarial data augmentation that is  label invariant.
%In the context of causal modeling, this is equivalent to causally invariant data augmentations~\cite{wang2022out}.
\section{Method}
We design the following method that interleaves exploration (stochastic gradient ascent) and exploitation (stochastic gradient descent) in order to extrapolate beyond the training data.
The exploration phase is motivated by Q-Learning \cite{watkins1992q}. This is a reinforcement learning method where an agent seeks to maximizes an expected reward. The agent takes a sequence of actions and collects rewards after each action. 

In Q-Learning, there is a Markov Decision Process (MDP) $\mathcal{M}=(\mathcal{S},\mathcal{A},p_t,p_r)$ consisting of a set of states, a set of actions that connect a state to a next state, a transition probability $p_t(s \xrightarrow[]{a} t)= P(t\mid s,a)$ for $s,t \in \mathcal{S}, a \in \mathcal{A}$ and a reward probability $p_r(r\mid s,a)$. If starting at $s$ there is an optimal expected reward, or \textbf{value} at $s$:  $V^*(s)$, then we define $Q^*(s,a)$ to be the expected reward when taking action $a$ starting at state $s$. In $Q$-learning, the agent computes a $Q(s,a)$ function over states and actions that estimates this optimal $Q^*$ function. The estimator can be learned through temporal updates. This is a dynamic programming recurrence called the Bellman-Equation \cite{watkins1992q}:
\begin{equation}
  Q_n(s,a)\gets  (1-\alpha)Q_{n-1}(s,a)+\alpha(r_{n}(s,a)+\gamma \max_{a' \in \mathcal{A}}Q_{n-1}(t,a')) \text{ where }p_t(s \xrightarrow[]{a} t)>0 
\end{equation}
where $n$ is the episode number.

Our method will use Q-Learning as an analogy for its explorative adversarially label invariant data augmentations.
\subsection{Relating Risk and Reward}
Consider the following ``analogy" conditioned over an environment $e$ between a MDP and deep learning:
\begin{enumerate}
    \item (States $\Longleftrightarrow$ Learners ): The set of states are in analogy with the set of learners. 
    \item (Actions $\Longleftrightarrow$ Weights $w \in W: \mathbb{A}_{w,e}$ ): 
    
    The weights $w \in W$ parameterize a distribution of \emph{label invariant} data augmentations. Let $\mathbb{A}_{w,e}$ be this distribution. Assume that the weight space $W$ is compact.
        \end{enumerate}
        By forming this analogy, the learners obtained through gradient updates correspond to states updated through actions. This lets us view the graph learning problem over a changing learner as a $Q$-learning problem.

    Continuing with the analogy, we relate the reinforcement learning expected reward with an ``augmented" risk. This ``augmentation" is a distribution over label invariant data augmentations parameterized by a weight $w \in W$.   
    \iffalse This is the pushforward distribution:
\begin{equation}
    (\mathbf{a})_{\#}(P^{e}):=P^{e} \circ {\mathbf{a}}^{-1}, \mathbf{a} \sim \mathbb{A}_{h,w}
\end{equation}
\fi
    \begin{enumerate}
    \setcounter{enumi}{2}
        \item (Reward $\Longleftrightarrow$ Risk over the Augmentations from (2))
        
    The reward at state-action pair $(h,w)$ is the risk augmented by $\mathbb{A}_{w,e}$:
\begin{equation}
    r^e(h,w):=\mathbb{E}_{\mathbf{a} \sim \mathbb{A}_{w,e}}[R^e(h\circ \mathbf{a})]
\end{equation}
     \end{enumerate}
     The Value function is thus analogous to maximization of the weight $w \in W$:
\begin{enumerate}
\setcounter{enumi}{3}
\item (The Value function $\Longleftrightarrow$ Maximum $w$) 
\begin{equation}\label{eq: w-max}
w_{\max}:= 
\argmax_{w }\mathbb{E}_{\mathbf{a} \sim \mathbb{A}_{w,e}}[R^e(h\circ \mathbf{a})]
\end{equation}
   \end{enumerate}
   We obtain the following for the relationship between the $Q$-function and the data augmentations in deep learning.
        \begin{lemma}(The Risk-Reward Analogy) \label{lemma: q-analogy}
     
     Assume $\alpha=1$. The $Q$-function in our analogy to deep learning must have $n=1$. Thus: 
     \begin{equation}\label{eq: Q-analogy}
         Q_1(h,w_{\max})\gets r^e(h,w_{\max})
     \end{equation}
     In our analogy, the $Q$-function is memory-less and exploitative. Analogously, the  average risk from deep learning is used for pure exploration.
     \end{lemma}
     \begin{proof}
         In deep learning, we can assume that the sequence of learners formed by SGD do not repeat due to stochasticity. Thus, we can assume that in the analogous $Q$-learning case, we are always in episode $n=1$. 
         
         Equation \ref{eq: Q-analogy} follows by $\alpha=1$. This does not use past states and maximizes the reward at its current state. 
         Analogously, in deep learning there is \emph{maximization} over the risk. Thus, the data augmentations are exploring for the learning process. 
     \end{proof}
The physical meaning of the $\argmax$ in Equation \ref{eq: w-max} is to skew the original data distribution  $P^e$ toward a pushforward distribution $(\mathbf{a})_{\#}(P^e)$ representing a ``hard" counterfactual distribution, where we measure hardness by the distance from the ERM loss over the training. In this context, the easiest possible data augmentations are just those that can reproduce the ERM loss. 
%This is achieved by a distribution of data augmentations supported only on the identity transformations. %We would like all transformed training data ${S}(\mathbf{G}), \mathbf{G}\sim P_{tr}, S \in \mathcal{S}_g$ and not just an arbitrary sample of them, to be explored. 

In other words, $\mathbb{A}_{w_{\max},e}$ is a distribution of data augmentations for environment $e$ that maximizes this hardness metric. This prevents collapse to an ERM solution. 
\subsubsection{Adversarial Counterfactual Distributions}
It would be presumed that by maximizing this hardness metric the augmentations from $\mathbb{A}_{w_{\max},e}$ can act as adversarial label invariant data augmentations in distribution through the risk. We call this an \textbf{adversarial counterfactual distribution}:
\begin{equation}\label{eq: p-aug}
    P^{\textsf{aug}(e)}:=(\mathbf{a})_{\#}(P^e): \mathbf{a} \sim \mathbb{A}_{w_{\max},e}, e \in \mathcal{E}_{tr}
\end{equation}
\begin{lemma}\label{lemma: aug-e}
    $P^{\textsf{aug}(e)}$ exists for any $e \in \mathcal{E}_{tr}$.
\end{lemma}
\begin{proof}
\textbf{1. }The distribution $\mathbb{A}_{w_{\max},e}$ is determined by Equation \ref{eq: w-max}. It exists since the space $W$ of weights for $\mathbb{A}_{\bullet,e}$ is compact. 

Let us simplify our data generation SCM to the causal path alone and denote $\textbf{Cause}$ for the causal variable(s) that by the map $m$ deterministically cause the label. 

\textbf{2. }Since this map is deterministic, the set of data samples $(x,y) \sim P^e$ form a deterministic map. 

Proof by contradiction:

Say $(x,y),(x,y')$ were a pair of covariate-label data pairs. These must have the same causation: $C$. Then by determinism of the map $m$, we must have $y=m(C)=y'$, contradiction.

\textbf{3. }
Because of the label invariance relation, we must have that $f=f\circ \mathbf{a}, \mathbf{a} \sim \mathbb{A}_{w_{\max},e}$. This means that the variable $\mathbf{a}(\mathbf{X}), \mathbf{a} \sim \mathbb{A}_{w_{\max},e}$ is caused by $\mathbf{X}$.

    Thus, we have an environment  $\textsf{aug}(e)$ that is the following chain: $(\mathbf{E}=e)\rightarrow \textbf{Cause}$. It generates the causal variable $\mathbf{X}$ and labels $\mathbf{Y}$ with $f\circ \mathbf{a}$.

    This can be summarized in the following diagram:
    \begin{equation}
        \begin{tikzcd}
	{\mathbf{E}=e} & {\textbf{Cause}} && \\
	{\mathbf{E}=\textsf{aug}(e)} & {\mathbf{X}} & {\mathbf{a}(\mathbf{X})} & {\mathbf{Y}}
	\arrow[from=1-1, to=1-2]
	\arrow[from=1-2, to=2-2]
	\arrow[curve={height=-12pt}, from=1-2, to=2-4]
	\arrow[from=2-1, to=2-2]
	\arrow[from=2-2, to=2-3]
	\arrow["f", curve={height=-12pt}, from=2-2, to=2-4]
	\arrow["{f\circ a}"', shift left, curve={height=12pt}, from=2-2, to=2-4]
	\arrow[from=2-3, to=2-4]
\end{tikzcd}
    \end{equation}
\end{proof}
The distribution $P^{\textsf{aug}(e)}$ can contain many instances where the output of the learner changes. This is not necessarily true over all instances, however.
\iffalse
This prevention is guaranteed since by definition of the $\mathbb{A}_{w_{\max},e}$ on $h_{ERM}$: 
  \begin{align}\label{eq: qmax-inequality}
\begin{split}
\mathbb{E}_{\mathbf{a} \sim \mathbb{A}_{w_{\max},e}(h_{ERM}, e,\mathcal{E}_{tr})}[R^e(h_{ERM}\circ \mathbf{a})]\geq R^e(h_{ERM}), \forall e \in \mathcal{E}_{tr}
\end{split}
\end{align}  
\fi
%This is achieved by maximizing a distribution on the data augmentations as in Definition \ref{def: G-Aug}. %By Proposition \ref{lemma: group-data-aug}, $\mathcal{S}_g=\cup_{n>0} S_g^n \subseteq\mathcal{I}_g$ when $S_g^n$ is a group action. An example of $\mathcal{S}_g$ is the set of edge additions/deletion data augmentations. 
%
 %where $d$ is a discrepancy measure such as mean squared error with $d(a,b)=0$ if $a=b$ and positive otherwise. We choose $d$ to be mean squared error since it is convex and thus has a unique minimum. 
  %Notice that we are choosing a skewed distribution on $\mathcal{S}_g$ to push the average regularizer value to a maximum for a given $\theta$.
 
% \begin{definition} \label{def:2}
%Let $P_{Aug}_{tr}(h) = P(({\mathbf{G}^e})'= \mathbf{a}(\mathbf{G}^e), \mathbf{Y}^e)$  satisfy $\mathbf{a} \sim \mathbb{A}_{w_{\max},e}(h,\{(\mathbf{G}^e,\mathbf{Y}^e)\}_{e \in \mathcal{E}_{tr}})$
%\end{definition}
%\textcolor{red}{
%That is, $P_{tr}_{Aug}(h)$ defines the perturbed data distribution that maximizes the ERM loss.
%}
\begin{figure}[!h]
\centering
\includegraphics[width=0.6\columnwidth]{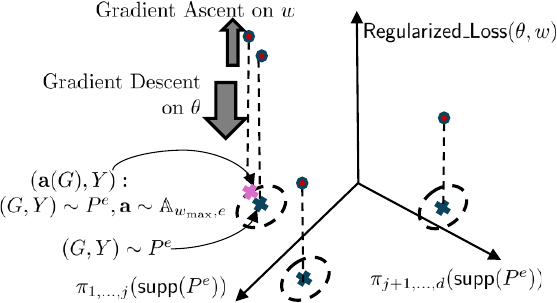}
		\caption{Geometric view of the minimax optimization procedure RIA algorithm on %$Reg_{\theta, w_k}= %\frac{1}{K} \sum_{k=1}^K
$\textsf{Regularized}\_\textsf{Loss}(\theta,w)$ as given in Equation \ref{eq: RIA} 
where $w \in W$ indexes the artificial search environments, $\theta$ indexes the learner's neural weights. The map $\pi_{j,...,k}: j<k$ is a projection of a set into $k-j+1$ independent dimensions. }
		\label{fig: geometric-optimization}
		\centering
	\end{figure}
\subsection{Regularization for Invariance with Adversarial Training: RIA}
We formulate the following minimax optimization problem called \textbf{Regularization for Invariance with Adversarial Training}: RIA. It uses the label invariance of existing causal learning methods with adversarial training. The data augmentations form an \textbf{adversarial counterfactual distribution} as in Equation \ref{eq: p-aug}. %with perturbed data distribution $P_{tr}_{Aug}(h)$. 
\begin{gather}\label{eq: RIA}
\begin{split}
\textsf{RIA}(\mathcal{E}_{tr})_{\bullet}\triangleq \min_{h\in \gH}  
\mathbb{E}_{((\mathbf{G}^e)',\mathbf{Y}^e)\sim P^{\textsf{aug}(e)}}[\lambda \cdot\textsf{OoD-Reg}_{\bullet}(h(({\mathbf{G}^e})'), \mathbf{Y}^e )
+l_e(h(({\mathbf{G}^e})'), \mathbf{Y}^e)] 
\\\text{where $P^{\textsf{aug}(e)}(h) = P[(\mathbf{a}(\mathbf{G}^e), \mathbf{Y}^e)]$  satisfies $\mathbf{a} \sim \mathbb{A}_{w_{\max},e}$}, \forall e \in \mathcal{E}_{tr}\text{, and }\lambda>0
\end{split}
\end{gather}
The subscript $\bullet$ indexes the constraints of some  OoD generalization method. 
\iffalse
\begin{gather}\label{eq: RIA}
h^*= \argmin_h 
%\mathbb{E}_{S \mathcal{S}_g}
\mathbb{E}_{ (G',Y) \sim P_{tr}_{Aug}} [l(h(G'), Y )] ,\\
\text{subject to }
\mathbb{E}_{(G',Y) \sim P_{tr}_{Aug}} [\max_{T  \in \mathcal{I}_g} d(h(G'), h(T(G'))]=0 \nonumber
\end{gather}
\fi
%\textcolor{red}{Please tell the audience what is OoD-Reg in (5). How is it relate to (4)? }

\textbf{Why Regularization? } 
In traditional OoD generalization methods, stabilization across environments imposes an invariance to a symmetry $\mathbf{a}$ over the data as a constraint for the learner $h$:
\begin{equation}
    h(\mathbf{a}(X))=h(X): X \sim P^e, e \in \mathcal{E}_{tr}
\end{equation}
This, however, prevents the data augmentation from being adversarial. Thus, in order to break the symmetry, we loosen this constraint and view the OoD generalization method through regularization.

We denote the regularization provided by existing OoD generalization methods by $\textsf{OoD-Reg}_{\bullet}(h)$.   The regularization maintains the original goal of stabilization across environments and extrapolation to an OoD test dataset.  If there is  ERM collapse, extrapolation cannot occur. The adversarially trained data augmentations help push the data away from ERM collapse.
\textcolor{black}{Intuitively, Equation~\ref{eq: RIA} aims to find the optimal OoD generalization classifier that minimizes the worst-case ERM loss, achieved via data augmentation. See Appendix Figure \ref{fig: ERMCollapse} for how this loss behaves during training and testing. 
%that is guaranteed not to converge to the ERM solution (you have NOT proved there is any guarantees)
}
\begin{theorem}(RIA can Escape ERM-Collapse)\label{prop: RIA-bounds}

When $\dagger$ is a constrained OoD generalization optimization problem with its risk denoted $\dagger(\mathcal{E}_{tr})$, we have:
    \begin{gather}
\begin{split}
\textsf{RIA}(\mathcal{E}_{tr})_{\dagger} \geq  {\dagger}(\mathcal{E}_{tr}) \geq \textsf{ERM}(\mathcal{E}_{tr})\geq 0
\end{split}
\end{gather}
Thus $\textsf{RIA}(\mathcal{E}_{tr})_{\dagger}$ can avoid ERM collapse.
\iffalse
We have 
\begin{equation}
    \textsf{OoD}(\gE_{all})= \textsf{RIA}(\gE_{tr})_{ERM} 
\end{equation}
if $\mathbb{A}_{w_{\max},e}$ solves for the data augmentation required by Corollary \ref{cor: OoD-by-advaug} for any learner $h$ during some minimization procedure.  
\fi
\end{theorem}
\begin{proof}
For the left inequality, by the temporal update rule from the $Q$-learning analogy in Lemma \ref{lemma: q-analogy}, that the $\mathbf{a} \sim \mathbb{A}_{w_{max},e}$ is risk maximizing. Thus:
  \begin{align}\label{eq: qmax-inequality}
\begin{split}
\mathbb{E}_{\mathbf{a} \sim \mathbb{A}_{w_{\max},e}(h_{\dagger})}[R^e(h_{\dagger}\circ \mathbf{a})]\geq R^e(h_{\dagger}), \forall e \in \mathcal{E}_{tr}
\end{split}
\end{align}  
%The far left inequality follows since the feasible region of learners is constrained so the 
where the left hand side is over the distribution $P^{\textsf{aug}(e)}$ which exists by Lemma \ref{lemma: aug-e}

For equality, if we set $\textsf{supp}(\mathbb{A}_{w_{\max},e}(h))= \{\textsf{id}\}$ then the minimizer of $\textsf{RIA}(\mathcal{E}_{tr})_{\dagger}$ in that case is the minimizer of $\dagger(\mathcal{E}_{tr}).$

%Expanding the support of $\mathbb{A}_{w_{\max},e}(h,\mathcal{E}_{tr})$ increases the risk on the left hand side of Equation \ref{eq: RIA-LB}, which results in the RIA loss.

The second inequality follows because there is a constraint of joint minimization in $\dagger$ but no such constraint in ERM. 

The last inequality follows because the risks are all non-negative.

The conclusion follows by the inequalities and the escape from Condition (3) for ERM collapse. Thus, by the contrapositive of Proposition \ref{prop: ERM-collapse-sufficient}, there is atleast one other environment. This gives the learner a chance to escape from ERM collapse.
\end{proof}
\subsection{Algorithm}
To solve the minimax optimization equation posed in Equation \ref{eq: RIA}, we propose an alternating gradient descent-ascent algorithm, which is shown in Algorithm~\ref{alg: altSGD}. The adversarial label invariant data augmentations are black-box~\cite{guo2019simple, zhang2024expressive}. On the contrary, in a white-box adversarial data augmentation, there would have to be the computationally expensive differentiation of a combinatorial object such as a graph. 

The algorithm proceeds in the form of a deep learning algorithm. The outer loop iterates epochs over the data. For $T$ steps, we compute the following three phases:
\begin{enumerate}
    \item Over all environments in $\mathcal{E}_{tr}$, a  random mask augmentation is computed over the graph. It is applied to the data.
    \item For all $T$ steps, the distribution $\mathbb{A}_{w,\bullet}$ is updated on parameter $w$ with (stochastic) gradient ascent.
    \item In one of the $T$ steps, the leaner $h_{\theta}$ is updated on parameter $\theta$ with (stochastic) gradient descent.
\end{enumerate}
In the algorithm, the GNN $f_{w}$, with neural weights $w$, determines a tensor of Bernoulli probabilities for which an adversarial data augmentation with $k$ entries is sampled. The GNN $h_{\theta}$ is some graph representation learner parameterized by $\theta$.

A geometric view of the optimization algorithm is shown in Figure \ref{fig: geometric-optimization}. In our implementation, we learn a distribution of node attribute masking data augmentations to prevent ERM collapse.

\textbf{The Choice of Data Augmentation: }
We chose a mask to augment the training data. The only requirement of RIA is that the data augmentation be label invariant. 

In Algorithm \ref{alg: altSGD}, the mask only applies to the node signal. Thus, it is \emph{spurious} to a ground truth graph classification that only depends on the graph topology. This is true, for example, in the datasets \textsc{CMNIST} and \textsc{Motif}. This means that the ground truth label over the graph topology is unaffected. Thus, our data augmentation is label invariant.

\begin{algorithm}[!h]
%\begin{algorithmic}
 \SetKwComment{Comment}{//}{}
\KwData{Training graph data $(G^e_i = (X^e_i, A^e_i)$,$Y^e_i$), $G^e_i \in D^e: D^e \sim (P^{e})^{n_e}$, $e \in \mathcal{E}_{tr}, i=1,...,n_e$; $n_e$ the number of training data for environment $e$. %$\mathcal{I}_{edge}^{inv}(\mathcal{G}_{X,A})$: the set of all edge addition/deletion data augmentations, 
Parameters of minimizing/maximizing GNN: $\theta$/$w$, Learning rates $lr_{\theta}$, $lr_{w}$, $k$: Number of entries of $X^e_i$ to keep, $\textbf{OoD-Reg}_{\bullet}$ is an OoD generalization regularizer from some existing method. $T$ is the ratio of num. maximization to num. minimization steps}

\While{not converged or max epochs not reached}{
\For{$t = 1,...,T$}{
\For{$e = 1,...,\lvert \mathcal{E}_{tr}\rvert $}{
%\For{$k = 1...K$}{
 $M_{w}^{e,i} \gets  s(\sigma((f_{w}(X^e_i
, A^e_i)))$; for $i=1,...,n_e$ \Comment{$f_{w}$ is a GNN; $s$ is a 0-1 sampler from a tensor of Bernoulli probs., sampling $k$ times to update a tensor of $0$'s.} %that sets the top k activations to 0 and all other entries to 1. 
 $G_{w}^{e,i} \gets ({M^{e,i}_{w}} \odot X^e_i,A^e_i)$
}
% $(X_e^k,A_e^k)\sim \mathbb{Q}_{w}(X^e, A^e), k = 1...K, e= 1...E_{tr}$
 $E(w,\theta) \gets \frac{1}{\lvert \mathcal{E}_{tr}\rvert }\sum_{e=1}^{\lvert \mathcal{E}_{tr}\rvert }\frac{1}{n_e} \sum_{i=1}^{n_e}[l_e(h_{\theta}, G^{e,i}_w, Y^e_i)]$
 
 $J(w,\theta) \gets$ 
$\frac{1}{\lvert \mathcal{E}_{tr}\rvert}\sum_{e=1}^{\lvert  \mathcal{E}_{tr}\rvert }\frac{1}{n_e}\sum_{i=1}^{n_e}[\textbf{OoD-Reg}_{\bullet}(h_{\theta},G_w^{e,i}, {Y}^e_i)]+E(w,\theta)$
 %\frac{\alpha}{P^{tr}_n}\sum_{i=1}^{n} \max_{T \in \mathcal{I}_{edge}^{inv}(\mathcal{G}_{X,A})} |( h_{\theta} \circ T(G_{w}^i)- h_{\theta}(G_{w}^i)|_2+ \frac{\beta}{E_{tr}} \sum_{i=1}^{n} l(h_{\theta}(G_{w}^i), Y^i)$
 %\\ 
 %   \Indpp{$+ \frac{1}{E_{tr}}\sum_{e=1}^{E_{tr}} \max_{T \in I_g} D( h_{\theta} \circ T(G_k^e), h_{\theta}(G_k^e))$}
 
  {Update $w \gets w +  lr_{w} \cdot \nabla_{w} E(w, \theta)$}%, k = 1...K$}
%}\ENDFOR

\If{t==T}{
 % loss $J_2(\theta) \gets \fracθ{\beta}{K \cdot E_{tr}} \sum_{k=1}^K \sum_{e=1}^{E_{tr}} l(h_{\theta}(G_k^e), Y)$
 %% $J_2(\theta) = {\beta} \cdot J_1(w, \theta)+\frac{\alpha}{K} \sum_{k=1}^K \sum_{e=1}^{E_{tr}} l(h_{\theta}(G_{w}^e), Y)$
%\frac{1}{K} \sum_{k=1}^K [ \frac{1}{E_{tr}}\sum_{e=1}^{E_{tr}} \sum_{T \in I_g} d( h_{\theta} \circ T(G_k^e), h_{\theta}(G_k^e))]$+\\
%\Indpp{$\frac{\beta}{K} \sum_{k=1}^K l(G_k^e, Y ; \theta)$ ; 
{Update $\theta \gets \theta - lr_{\theta} \cdot \nabla_{\theta} J(w,\theta)$ ;}
}
}
}
%\end{algorithmic}
\caption{\color{black}{RIA by Alternating (Stochastic) Gradient Ascent-Descent with Adversarial Data Augmentation for OoD Generalization on Graphs}}\label{alg: altSGD}
\end{algorithm}
\section{Experiments}\label{sec: experiments}
We ran all our experiments on a 64 core Intel(R) Xeon(R) CPUs @2.40 GHz with 128 GB DRAM equipped with one 40 GB DRAM Ampere A100 GPU. %For each dataset experiment, all runs completed within 2 days. 
The corresponding test scores for the best in-distribution validation score are averaged across $3$ runs for both real world and synthetic datasets. Hyperparameters follow the defaults of the GOOD benchmark~\cite{gui2022good}, see the Appendix.

We implement Algorithm \ref{alg: altSGD} (referred to as {\sc RIA} in 
Table~\ref{tab: allexps}) using the regularizations of RICE, IRM, VREx. 
%We also implement a variation where the adversarial data augmentations are given by an injective normalizing flow network of the form $\theta_{\xi}^X(X)= X \odot exp(s_{\xi_1}(X))+t_{\xi_2}(X)$ where $\xi= [\xi_1,\xi_2]$ as given by \cite{dinh2016density} on the node attributes $X$ of input data $(X,A)$. We call this approach {\sc ours-RICE-NF}, shown in Table \ref{tab: allexps}. %We perform this experiment because we surmise that $\mathcal{M}_{c_*}$ is made up of invertible maps. Notice that when there is no gradient ascent, $lr_w=0$ in Algorithm \ref{alg: altSGD}, then our algorithm is equivalent to RICE. This is how we implement RICE.
%
We compare our approach with the baselines of Coral~\cite{sun2016deep}, DANN~\cite{ganin2016domain}, DIR~\cite{wu2022discovering}, ERM~\cite{vapnik1999overview}, GSAT~\cite{miao2022interpretable}, GroupDRO~\cite{sagawa2019distributionally}, IRM~\cite{arjovsky2019invariant}, Mixup~\cite{wang2021mixup}, RICE~\cite{wang2022out}, VREx~\cite{krueger2021out}, EdgeDrop~\cite{Rong2020DropEdge:} all implemented in the GOOD~\cite{gui2022good} benchmark.

For the following datasets the graph data $G$ is split between signal $X$ and topology $A$. Since the signal is spurious for the graph classification task for our datasets, we naturally have a disentanglement between causal and spurious parts of the graph. This allows us to define causally invariant data augmentations on the data as perturbations on the signal $X$. This is one of the reasons why our theory is designed for graphs. Images do not have a natural tensor disentanglement such as between foreground and background without labels. 

\begin{table*}[!h]
\centering
\noindent
% \usepackage{booktabs}
% \usepackage{booktabs}
\resizebox{1\columnwidth}{!}{
\begin{tabular}
{@{}lll|ll|ll|ll|ll|ll|ll|}
  \scriptsize Dataset (acc) & \multicolumn{2}{c}{\scriptsize {\sc CMNIST }$\uparrow$}  & \multicolumn{2}{c}{\scriptsize {\sc SST2}$\uparrow$}   & \multicolumn{4}{c}{\scriptsize {\sc Motif} $\uparrow$}                            & \multicolumn{4}{c}{\scriptsize {\sc AMotif}$\uparrow$}  & \multicolumn{2}{c}{\scriptsize {\sc Synth} $\uparrow$}                         \\
\cmidrule(lr){2-3}
          \cmidrule(lr){4-5}
          \cmidrule(lr){6-9}
          \cmidrule(lr){10-13}
          \cmidrule(lr){14-15}
          \scriptsize covariate & \multicolumn{2}{c}{\scriptsize color}  & \multicolumn{2}{c}{\scriptsize length} & \multicolumn{2}{c}{\scriptsize basis} & \multicolumn{2}{c}{\scriptsize size} & \multicolumn{2}{c}{\scriptsize basis} & \multicolumn{2}{c}{\scriptsize size} & \multicolumn{2}{c}{\scriptsize basis+std, $r=1$}\\
          \cmidrule(lr){2-3}
          \cmidrule(lr){4-5}
          \cmidrule(lr){6-7}
          \cmidrule(lr){8-9}
          \cmidrule(lr){10-11}
          \cmidrule(lr){12-13}
          \cmidrule(lr){14-15}
          & \multicolumn{1}{c}{\tiny ID}                & \multicolumn{1}{c}{\tiny OOD}    & \multicolumn{1}{c}{\tiny ID}              & \multicolumn{1}{c}{\tiny OOD}     & \multicolumn{1}{c}{\tiny ID}               & \multicolumn{1}{c}{\tiny OOD}    & \multicolumn{1}{c}{\tiny ID}         & \multicolumn{1}{c}{\tiny OOD}         & \multicolumn{1}{c}{\tiny ID}               & \multicolumn{1}{c}{\tiny OOD}    & \multicolumn{1}{c}{\tiny ID}         & \multicolumn{1}{c}{\tiny OOD} & \multicolumn{1}{c}{\tiny ID}         & \multicolumn{1}{c}{\tiny OOD}    \\

  %\hline
\scriptsize {\sc \lightgrayarea{RIA-RICE}} & \tiny 61.7$\pm$1.6 &	\tiny \red{$48.1\pm 0.8$} &	\tiny 89.4$\pm$0.6 &	\tiny \red{$81.9 \pm 0.2$} &	\tiny \gray{$92.4\pm 0.2$} &	\tiny 65.1$\pm$5.9 &	\tiny \red{$92.4 \pm 0.2$} &	\tiny \red{$55.3\pm 0.4$} &	\tiny 79.3$\pm$1.6 &	\tiny 36.8$\pm$4.2 &	\tiny 67.4$\pm$1.5 &	\tiny 33.4$\pm$1.3 &	\tiny 48.0$\pm$9.0 &	\tiny 58.5$\pm$1.5 \\
%\hline
\scriptsize {\sc \lightgrayarea{RIA-IRM}} &
  \tiny {$65.5 \pm 2.8$} &
  \tiny \gray{$41.6 \pm 0.6$} &
  \tiny {$89.7 \pm 0.6$} &
  \tiny \gray{$81.7 \pm 0.5$} &
  \tiny {$33.7 \pm 0.8$} &
  \tiny {$33.9 \pm 0.7$} &
  \tiny {$33.5 \pm 0.8$} &
  \tiny {$34 \pm 2.9$} &
  \tiny \gray{$89.6 \pm 0.8$} &
  \tiny \gray{$40.5 \pm 3.8$} &
  \tiny {$48.6 \pm 0.4$} &
  \tiny \red{$48.6 \pm 2$} &
  \tiny {$51 \pm 0.6$} &
  \tiny {$54 \pm 0.8$} \\
  %\hline
\scriptsize {\sc \lightgrayarea{RIA-VREx}} &
  \tiny \red{$79.3 \pm 0.7$} &
  \tiny {$38.7 \pm 0.7$} &
  \tiny \gray{$89.8 \pm 2$} &
  \tiny {$80.2 \pm 4$} &
  \tiny {$32.2 \pm 2.3$} &
  \tiny {$34 \pm 1.7$} &
  \tiny {$33.5 \pm 0.5$} &
  \tiny {$34 \pm 1.0$} &
  \tiny \red{$90.5 \pm 4.5$} &
  \tiny \red{$42.4 \pm 0.6$} &
  \tiny \red{$90.3 \pm 0.9$} &
  \tiny \gray{$47 \pm 0.87$} &
  \tiny {$40 \pm 0.8$} &
  \tiny \gray{$60 \pm 1.9$} \\
  %\hline
%\scriptsize {\sc \lightgrayarea{ours-RICE-NF}}   & \tiny 20.3$\pm$0.3 &	\tiny 21.0$\pm$0.0 &	\tiny 66.2$\pm$1.5 &	\tiny 62.7$\pm$5.9 &	\tiny 92.4$\pm$0.2 &	\tiny 65.1$\pm$5.9 &	\tiny \gray{$92.3\pm 0.1$} &	\tiny \gray{$55.2\pm 0.3$} &	\tiny {82.9$\pm$7.9} &	\tiny \red{$44.4 \pm 5.2$} &	\tiny \gray{$85.0\pm 0.3$} &	\tiny {36.8$\pm$0.1} &	\tiny 52.5$\pm$0.5 &	\tiny {61.0$\pm$1.0}    \\
\hline
\scriptsize {\sc ERM}       &  \tiny {77.5$\pm$0.5} &	\tiny 28.3$\pm$0.3 &	\tiny {89.4$\pm$0.4} &	\tiny 81.2$\pm$0.2 &	\tiny 92.3$\pm$0.3 &	\tiny 68.3$\pm$0.3 &	\tiny 92.1$\pm$0.1 &	\tiny 51.4$\pm$0.4 &	\tiny 80.8$\pm$1.1 &	\tiny 33.2$\pm$1.0 &	\tiny 67.9$\pm$2.2 &	\tiny 33.2$\pm$1.0 &	\tiny 53.5$\pm$1.5 &	\tiny 53.5$\pm$1.5    \\
\scriptsize { \sc {DIR}} &
  \tiny {$39 \pm 2.9$} &
  \tiny {$28.1 \pm 10$} &
  \tiny {$83.6 \pm 4.6$} &
  \tiny {$81.1 \pm 4.9$} &
  \tiny {$82.2 \pm 5.2$} &
  \tiny \red{$73.6 \pm 5.8$} &
  \tiny {$75.6 \pm 3.9$} &
  \tiny {$39.3 \pm 1$} &
  \tiny {$34.7 \pm 2.5$} &
  \tiny {$35 \pm 2.9$} &
  \tiny {$36.3 \pm 5.2$} &
  \tiny {$33.1 \pm 3.3$} &
  \tiny {$48 \pm 1.2$} &
  \tiny \red{$61 \pm 1.4$} \\
\scriptsize  {\sc RICE} &  \tiny 68.2$\pm$0.9 &	\tiny 26.3$\pm$0.5 &	\tiny \red{$90.0\pm 0.2$} &	\tiny 80.7$\pm$0.7 &	\tiny 92.4$\pm$0.2 &	\tiny 65.1$\pm$5.9 &	\tiny \gray{92.2$\pm$0.0} &	\tiny \gray{55.1$\pm$0.2} &	\tiny 69.3$\pm$9.8 &	\tiny 36.2$\pm$1.7 &	\tiny 50.5$\pm$9.2 &	\tiny 33.5$\pm$1.2 &	\tiny 54.5$\pm$2.5 &	\tiny 54.0$\pm$1.0        \\
\scriptsize {\sc Coral}     &  \tiny \gray{$78.3 \pm 0.3$} &	\tiny 29.0$\pm$0.0 &	\tiny 89.3$\pm$0.3 &	\tiny 79.4$\pm$0.4 &	\tiny 92.3$\pm$0.3 &	\tiny 68.4$\pm$0.4 &	\tiny 92.1$\pm$0.1 &	\tiny 50.5$\pm$0.5 &	\tiny 81.0$\pm$0.2 &	\tiny 33.9$\pm$1.3 &	\tiny 67.9$\pm$0.6 &	\tiny 32.9$\pm$0.8 &	\tiny 54.0$\pm$2.0 &	\tiny 51.5$\pm$2.5    \\
\scriptsize {\sc DANN}      &   \tiny 77.5$\pm$0.5 &	\tiny 29.1$\pm$0.6 &	\tiny 89.3$\pm$0.8 &	\tiny 79.4$\pm$0.9 &	\tiny 92.3$\pm$0.8 &	\tiny 65.2$\pm$0.7 &	\tiny 92.1$\pm$0.6 &	\tiny 51.2$\pm$0.7 &	\tiny {81.1$\pm$0.2} &	\tiny 38.1$\pm$1.4 &	\tiny 69.2$\pm$1.1 &	\tiny 33.1$\pm$0.5 &	\tiny 54.5$\pm$1.8 &	\tiny 52.0$\pm$0.5      \\
\scriptsize {\sc GroupDRO}  &   \tiny 77.0$\pm$1.0 &	\tiny 28.5$\pm$0.5 &	\tiny 88.8$\pm$0.8 &	\tiny 80.7$\pm$0.7 &	\tiny 91.8$\pm$0.8 &	\tiny 67.6$\pm$0.6 &	\tiny 91.6$\pm$0.6 &	\tiny 51.0$\pm$1.0 &	\tiny 74.0$\pm$1.0 &	\tiny {38.6$\pm$0.6} &	\tiny \gray{83.9$\pm$0.8} &	\tiny {35.8$\pm$0.8} &	\tiny 50.5$\pm$0.5 &	\tiny 52.5$\pm$0.5     \\
\scriptsize {\sc GSAT}      &\tiny 67.0$\pm$2.6 &	\tiny {39.9$\pm$0.6} &	\tiny 89.0$\pm$0.1 &	\tiny 80.6$\pm$1.1 &	\tiny \red{$92.5 \pm 0.0$} &	\tiny 57.1$\pm$6.8 &	\tiny 92.1$\pm$0.1 &	\tiny 53.3$\pm$0.3 &	\tiny 69.3$\pm$9.8 &	\tiny 36.2$\pm$1.7 &	\tiny 50.5$\pm$9.2 &	\tiny 33.5$\pm$1.2 &	\tiny {58.5$\pm$7.5} &	\tiny 50.5$\pm$6.5  \\
\scriptsize {\sc IRM}       &  \tiny 77.0$\pm$1.0 &	\tiny 26.9$\pm$0.9 &	\tiny 88.7$\pm$0.7 &	\tiny 79.0$\pm$1.0 &	\tiny 91.8$\pm$0.8 &	\tiny {69.8$\pm$0.8} &	\tiny 91.6$\pm$0.6 &	\tiny 50.9$\pm$0.9 &	\tiny 79.0$\pm$1.0 &	\tiny 37.9$\pm$0.9 &	\tiny 79.6$\pm$0.6 &	\tiny 33.6$\pm$0.6 &	\tiny \red{$62.5 \pm 0.5$} &	\tiny 48.5$\pm$0.5   \\
\scriptsize {\sc Mixup}     & \tiny 76.7$\pm$0.7 &	\tiny 25.7$\pm$0.7 &	\tiny 88.9$\pm$0.9 &	\tiny 79.9$\pm$0.9 &	\tiny 91.8$\pm$0.8 &	\tiny 69.5$\pm$0.5 &	\tiny 91.5$\pm$0.5 &	\tiny 50.7$\pm$0.7 &	\tiny 70.9$\pm$0.9 &	\tiny 36.7$\pm$0.7 &	\tiny 68.7$\pm$0.7 &	\tiny 33.0$\pm$1.0 &	\tiny 41.5$\pm$0.5 &	\tiny {58.5$\pm$0.5}   \\
\scriptsize {\sc VREx}      &  \tiny 77.0$\pm$1.0 &	\tiny 27.7$\pm$0.7 &	\tiny 88.8$\pm$0.8 &	\tiny 79.8$\pm$0.8 &	\tiny 91.8$\pm$0.8 &	\tiny \gray{$70.7 \pm 0.7$} &	\tiny 91.6$\pm$0.6 &	\tiny 51.8$\pm$0.8 &	\tiny 78.6$\pm$0.6 &	\tiny 33.9$\pm$0.9 &	\tiny 65.6$\pm$0.6 &	\tiny 34.0$\pm$1.0 &	\tiny 50.5$\pm$0.5 &	\tiny 52.5$\pm$0.5 \\
\scriptsize{\sc DropEdge} & \tiny 56.9$\pm$0.9 &	\tiny 19.7$\pm$0.7 &	\tiny 88.8$\pm$0.8 &	\tiny {81.7$\pm$0.7} &	\tiny 34.7$\pm$0.7 &	\tiny 31.5$\pm$0.5 &	\tiny 34.8$\pm$0.8 &	\tiny 31.6$\pm$0.6 &	\tiny 37.9$\pm$0.9 &	\tiny 33.9$\pm$0.9 &	\tiny 33.8$\pm$0.8 &	\tiny 33.0$\pm$1.0 &	\tiny \gray{$59.5 \pm 0.5$} &	\tiny 43.5$\pm$0.5 
\end{tabular}}
\caption{Accuracy of all baseline approaches as well as \protect{\sc RIA-RICE}, \protect{\sc RIA-IRM}, \protect{\sc RIA-VREx} on all datasets under different covariate shifts. For each covariate shift, the columns labeled ID refer to the in-distribution test accuracies while the columns labeled OOD refer to the out-of-distribution test scores. \protect\cfbox{red}{Red} and \protect\cfbox{gray}{gray} entries are the max and second max test accuracies, respectively, for each column.}\label{tab: allexps}
\end{table*}

%\subsection{Modeling Causal Noise Synthetically}
\textbf{Additive Spurious Attributes Synthetic Dataset:} We develop a synthetic binary classification dataset that models a noisy  data generation process as in the SCM in Appendix Figure~\ref{fig: SEM}. For more information on the dataset, see Appendix, section \ref{sec: hyperparamsdatasets}. It is designed to model attribute shifts instead of just shifts in the graph topologies as in {\sc Motif}. %We assume there is a causal and spurious set of attributes $X_C,X_S$ that induce the input graph $G=(X,A)$, and there exists a map $c_*$ from the data back to the causal attributes $X_C$. The labels $Y$ are generated from a readout function on $(X_C,A)$. We also assume there is a map $c_{\xi}= c_* +s_{\xi}$ exists as an invertible map from $(X,A)$ to $X_C+X_S$. The causal graph $(X_C,A)$ can be obtained by subtracting $X_S$ from $X_C+X_S$. The joining operation, is thus now $(X,A):= (J_X(X_C,X_S), J_A(X_C,X_S))= c_{\xi}^{-1}(X_S+X_C)$. % to the additive combination of He initialization is defined by $\xi \sim \Xi= U(-\sqrt{\frac{1}{f}}, \sqrt{\frac{1}{f}})$, %the uniform distribution on the half open interval $[-\sqrt{\frac{1}{f}}, \sqrt{\frac{1}{f}})$, where $f$ is the number of attributes of each node. The random variable $\xi$ is independent to the covariates. 
%The covariate shift is modeled by using the basis structural covariate shift as given by {\sc Motif} and setting each attribute of $X_S$ to come from a normal distribution. % See Table \ref{tab: allexps} for all accuracy scores on the synthetic dataset. %The accuracy is high as expected since the data generation assumptions relate to our method.

%We perform graph classification for this dataset where the OoD generalization testing accuracy is shown in Table REF. 

%Figure \ref{fig: heatmap-rxnsxacac} shows a plot of the normalized OoD test accuracies of the synthetic dataset of our method. We notice that in general the smaller the value of causal radius $r$ from the data generating process, the less effective are data augmentations that affect the data more (more entries to zero).
%\subsubsection{Measuring the Causal Noise Radius $r$}

%{measure the $\varepsilon$ adversarial robustness value after training as a function of $r$}

\textbf{Real World Graph Classification Experiments:} We also perform experiments on real world benchmarks. For all the scores, see Table \ref{tab: allexps}. We use the datasets of {\sc CMNIST}~\cite{arjovsky2019invariant}, {\sc SST2}~\cite{JMLR:v22:21-0343}, and {\sc Motif}~\cite{wu2022discovering} from the GOOD framework as well as {\sc AMotif}, a modification of {\sc Motif}. Each of these datasets follows the causal model as shown in Appendix Figure \ref{fig: SEM}. Accuracy is used to measure the performance on all the datasets, as is standard. %Since {\sc Motif} has no node attributes, we create a variation of {\sc Motif} called {\sc AMotif} which has node attributes sampled from a normal distribution with various standard deviations. 
Each dataset involves different kinds of covariate shift. For more details about each dataset and the kind of covariate shift imposed on them, see the Appendix. 

As shown in Table~\ref{tab: allexps}, our method, RIA, performs well both in the in-distribution ID and out-of-distribution OoD settings.  For the ID case, 
RIA performs the {\color{red}{highest}} or {\color{gray}{second highest}} on all datasets in at least one method except for the synthetic dataset. This suggests that even in the ID setting, the data is never truly in-distribution. There is always some benefit to pushing away from the ERM solution. For the OoD case, the adversarial data augmentations seem able to counterfactually generate environments similar to the testing input data. This is the  benefit to minimax optimization. Of course there is no guarantee that RIA is converting the training distribution into the testing distribution exactly. However, the training distribution is no longer the same thing. RIA obtains the {\color{red}{highest}} or {\color{gray}{second highest}} score for every dataset except {\sc Motif} by at least one method. The performance on {\sc Motif} is not high since {\sc Motif} has very simple attributes.  The ablation comparison between each existing method: IRM, RICE, VREx, and RIA applied to it are included in Table~\ref{tab: allexps}. We see that RIA not only improves upon the existing method, but oftentimes outperforms many other baselines. 
%This is why we introduce the {\sc AMotif} dataset. %Since in the OoD case the solution cannot be an ERM solution, our method can guarantee any existing regularization based OoD method does not collapse to an ERM solution. 

%The validation and ID test losses are not shown. They are both decreasing as in the left of Figure \ref{fig: ERMCollapse}. Surprisingly, the ID test loss performs better for RIA. 

\subsection{Illustrating ERM Collapse}
In Figure \ref{fig: ERMCollapse}, we show the training and OoD testing losses across 150 epochs of training for ERM, IRM and VREx as well as RIA applied to IRM and VREx. We can see the ERM collapse phenomenon. SST2 does not have as much of a distribution shift so it is harder to observe ERM collapse. \textsc{CMNIST}has a synthetic distribution shift attached to a natural data distribution and only two very similar training environments so it is easier to observe ERM collapse. On \textsc{CMNIST}, VREx and IRM both follow the training loss curve of ERM since they must converge to zero training loss. RIA-VREx and RIA-IRM, on the other hand, are prevented from converging to zero loss. %On SST2 the variance regularizer makes the training more erratic. 
For OoD generalization for both SST2 and \textsc{CMNIST}, we see that by preventing ERM collapse, we can in fact maintain low OoD loss and prevent mimicking the behavior of ERM. The other methods, IRM and VREx, on the other hand, diverge like ERM.% and then come back down in loss but do not drop below their initialization.      

%The validation and ID test losses are not shown. They are both decreasing as in the left of Figure \ref{fig: ERMCollapse}. Surprisingly, the ID test loss performs better for RIA.  
\begin{figure}[h]%
\begin{subfigure}
    {{\includegraphics[width=0.5\textwidth]{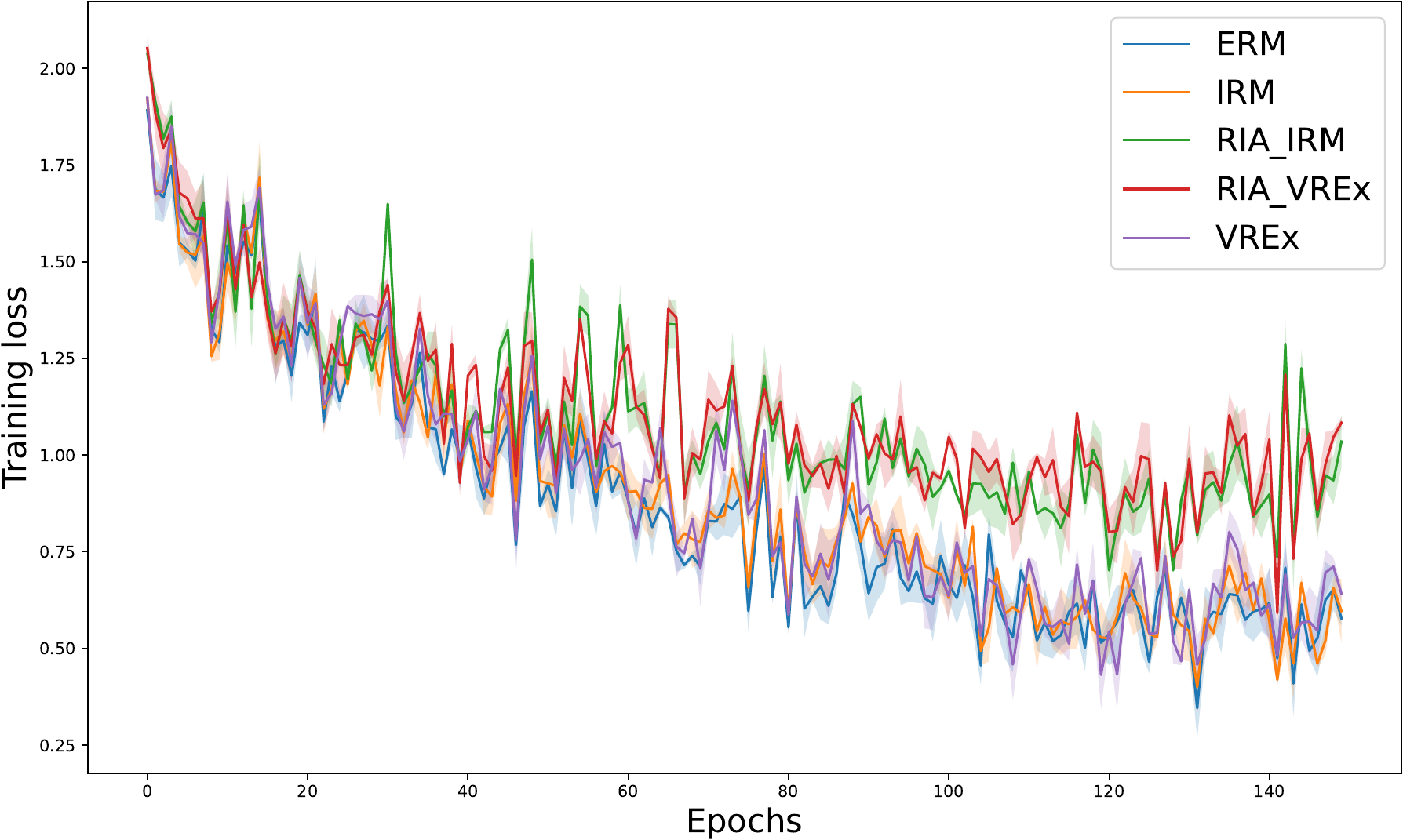} }}%
    \label{subfig: trainloss}
\end{subfigure}
    \begin{subfigure}
    {{\includegraphics[width=0.5\textwidth]{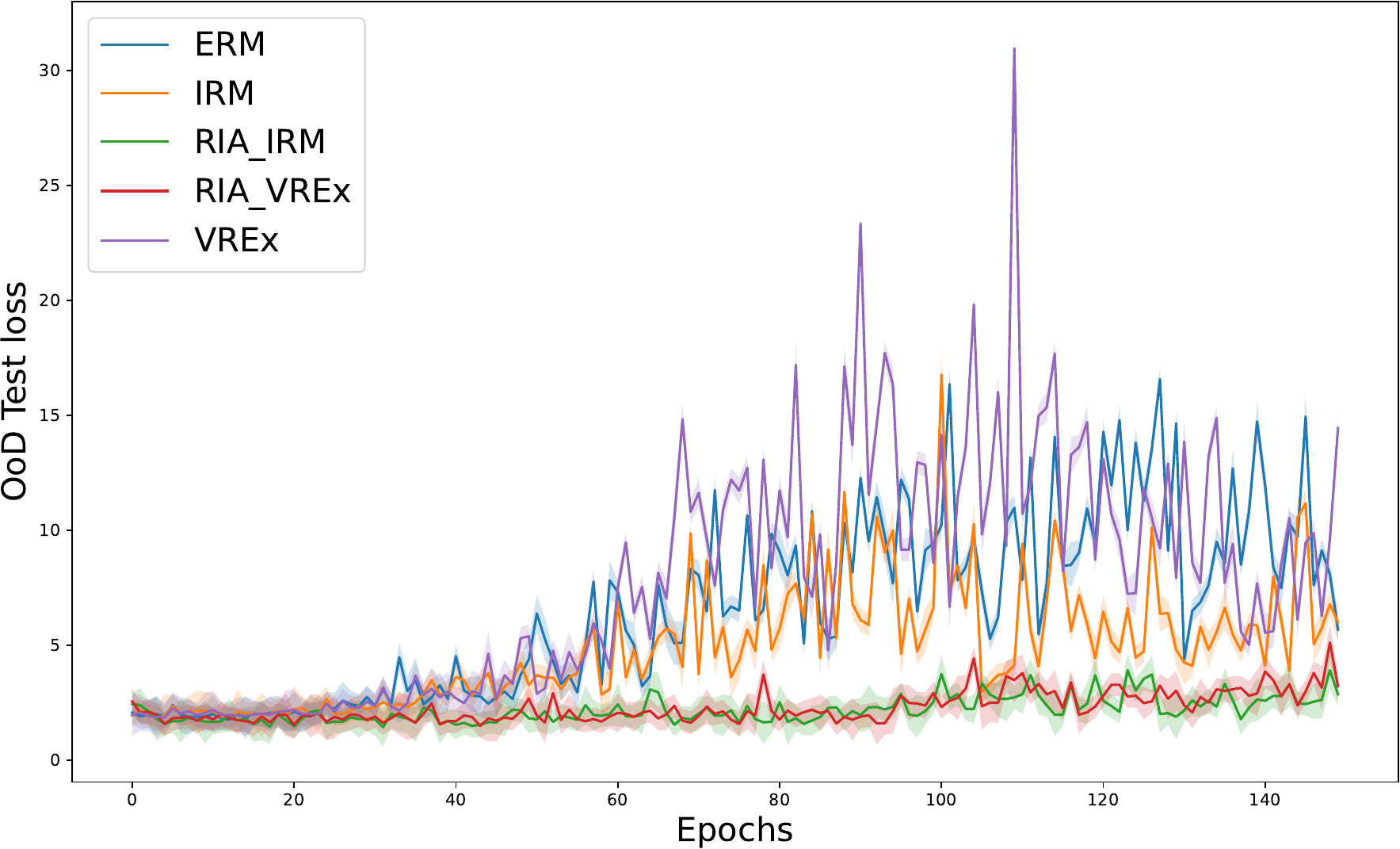} }}%
    \label{subfig: testOoDloss}
    \end{subfigure}
    \begin{subfigure}
    {{\includegraphics[width=0.5\textwidth]{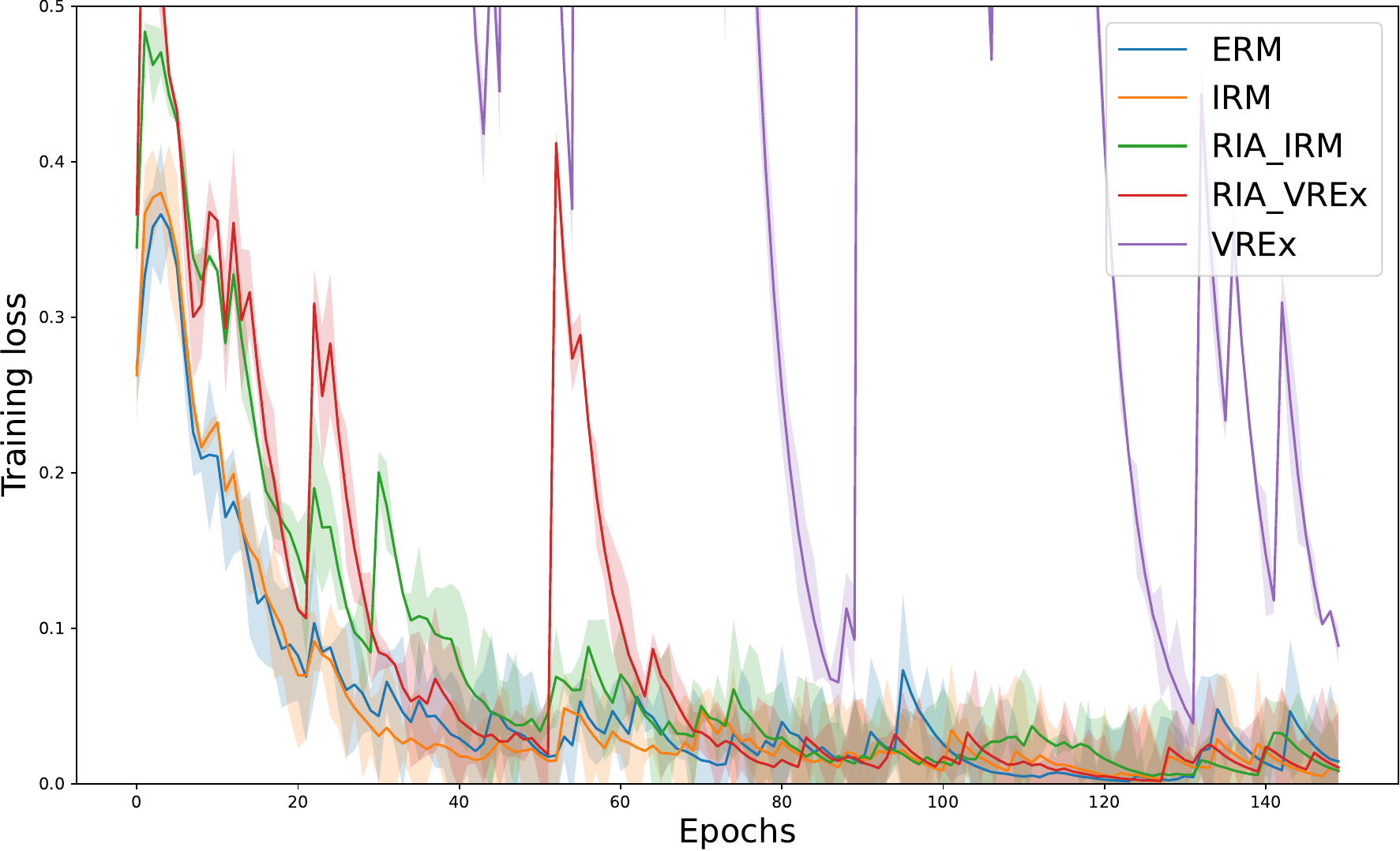} }}%
    \label{subfig: trainloss-SST2}
    \end{subfigure}
    \begin{subfigure}
    {{\includegraphics[width=0.5\textwidth]{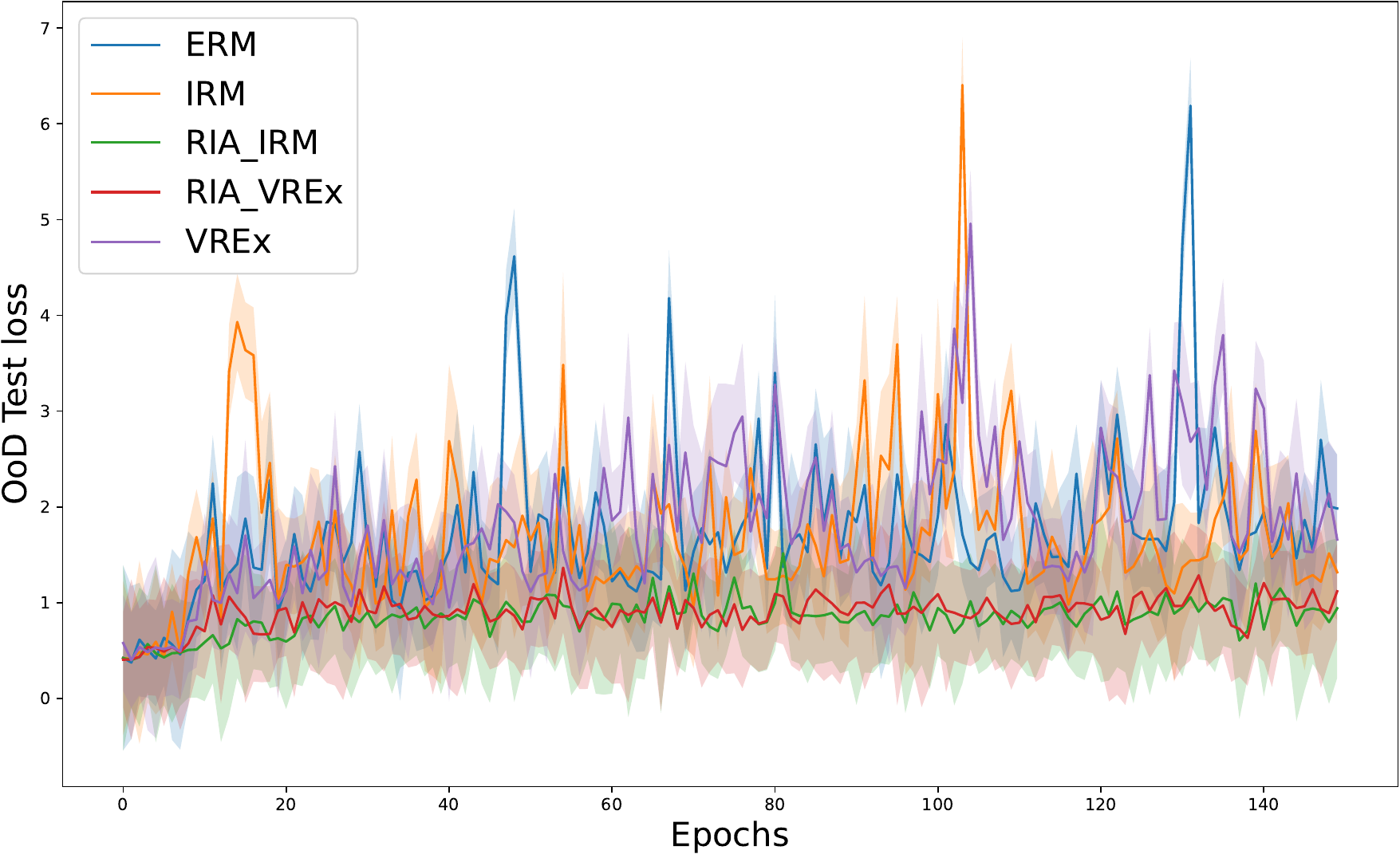} }}%
    \label{subfig: testOoDloss-SST2}
    \end{subfigure}
    \caption{Illustration of ERM Collapse on the \textsc{CMNIST}(above) and SST2 (below) dataset. Left: Training loss where ERM collapse is happening to traditional constrained optimization OoD generalization methods. Red and Green are RIA on IRM and VRex, respectively. Right: Test OoD loss. The consequences of ERM collapse are prevented. }%
    \label{fig: ERMCollapse}%
\end{figure}
\section{Discussion}
We observe widespread ERM collapse in existing methods in our experiments. Many of the methods such as IRM, VREx, Mixup and DropEdge behave very similar to ERM. We believe that these particular methods do not veer from ERM aggressively enough. IRM and VREx, may not have enough training environments. Mixup and DropEdge, as static data augmentations, are not actually changing the training distribution or achieving any kind of invariance across environments. RIA prevents ERM collapse and due to the adversarial generation of environments against the ERM loss the learner has enhanced robustness. 

Although we only did experiments on graph data, we believe RIA can easily be implemented for images and other data modalities. One caveat we have observed empirically is that the data augmentations should be diverse and only slightly affect the training distribution. Sudden changes to the training distribution can over-correct the learner.  
\section{Conclusion}
We have introduced adversarial data augmentations to provide a search for a robust OoD solution. We formulate and motivate the OoD problem as a minimax optimization problem over a set of environments. To address the lack of training environments and to prevent an early collapse of the classifier onto an ERM solution on the training distribution during OoD training, we propose RIA: Regularization for invariance with adversarial training. 
We compare our approach, RIA, with state of the art OoD generalization approaches including DIR\cite{wu2022discovering} and RICE \cite{wang2022out} as well as the classical ERM on graphs. This shows that for graph classification, preventing ERM collapse in the OoD setting improves existing OoD generalization methods. %We perform extensive experiments showing the effectiveness of our approach on OoD datasets and even in-distribution datasets. 

\section*{Acknowledgment}%ACKNOWLEDGMENT} 
This work was supported in part by the National Science Foundation under Grant OAC-$2310510$. 

\bibliography{bibliography}
\appendix
\clearpage
\section{The Regularizer for OoD Generalization Methods}
We have identified three OoD generalization methods that are formulated as constrained optimization problems: IRM, VREx, and RICE. We go over each method and how they can be rewritten as regularized ERM methods. Regularized ERM methods risk the possiblity of ERM collapse since their constraints may fail to be effective. 

Let $R^e(\bullet)$ denote the risk function over a given environment $e \in \mathcal{E}_{tr}$.

We will compute the regularizer for a constrained OoD generalization method $(\dagger)$. The regularizer is dependent on a learner $h_{\theta}$ and averaged over the an environment $e \in \mathcal{E}_{tr}$. This is denoted:
\begin{equation}
    \textbf{OoD-Reg}_{\dagger}(h_{\theta},\textsf{supp}(P^e))
\end{equation}

\textbf{IRM}: IRM is the following optimization problem:

\begin{gather}
    \begin{split}
        \min_{\Phi:X \rightarrow H,w: H\rightarrow Y} \sum_{e\in \mathcal{E}_{tr}} R^e(w \circ \Phi) \\
\text{ s.t. } w \in \argmin_{w:H \rightarrow Y} R^e( w \circ \Phi), \forall e \in \mathcal{E}_{tr}
    \end{split}
\end{gather}
This can be written as the following regularized ERM problem called IRMv1. 
\begin{equation}
    \min_{\Phi}\sum_{e \in \mathcal{E}_{tr}}(R^e(\Phi)+\lambda\cdot (\lvert \nabla_{\bar{w}} \mid_{\bar{w}=1.0} R^e(\bar{w} \cdot  \Phi)\rvert_2)^2)
\end{equation}
The regularizer is thus the following:
\begin{equation}\textbf{OoD-Reg}_{IRM}(h_{\Phi},\textsf{supp}(P^e)):= \lvert \mathcal{E}_{tr}\rvert \cdot (\lvert \nabla_{\bar{w}} \mid_{\bar{w}=1.0} R^e(\bar{w} \cdot  \Phi)\rvert_2)^2
\end{equation}
For graph learning, the map $\Phi$ can be implemented as a graph representation learner such as a GNN. The $w$ learnable scalar parameter just multiplies the representation before taking the cross entropy loss.

One can check that the causal model of Section \ref{sec: causalmodel} is still compatible with IRM.

\textbf{VREx}: VREx is the following optimization problem:

\begin{gather}
    \begin{split}
        \max_{
\sum_{e \in \mathcal{E}_{tr}} \lambda_e=1,
\lambda_e \geq \lambda_{min}}
\sum_{e \in \mathcal{E}_{tr}}
\lambda_e \cdot R^e(h_{\theta})=\\ (1 - \lvert \mathcal{E}_{tr}\rvert \cdot \lambda_{min}) \cdot \max_{e \in \mathcal{E}_{tr}} R^e(h_{\theta}) + \lambda_{min}  \cdot \sum_{e \in \mathcal{E}_{tr}}
R^e(h_{\theta})
    \end{split}
\end{gather}
This can be approximated as the following regularized ERM problem called VREx whose minimization gives a smoother version of the MM-REx constrained optimization problem:
\begin{equation}
\beta \cdot \text{Var}(\{R^1(h_{\theta}), ..., R^{\lvert \gE_{tr}\rvert }(h_{\theta})\}) + \sum_{e \in \mathcal{E}_{tr}}
R^e(h_{\theta})
\end{equation}
The regularizer for V-REx can then be defined as:
\begin{equation}
    \textbf{OoD-Reg}_{IRM}(h_{\theta},\textsf{supp}(P^e)):=\lvert \mathcal{E}_{tr}\rvert\cdot (R^e(h_{\theta}))^2
\end{equation}
The implementation for VREx on graphs should be straight forward since it is just a new regularized loss for a graph representation learner.

\textbf{RICE}: %RICE is the following optimization problem:
\iffalse
\begin{gather}
    \begin{split}
        H_s = \argmin
_{h}
\mathbb{E}_{P_s} [l(h(X), Y )]\\
\text{ s.t.} h(\cdot ) = (h \circ T )(\cdot), \forall T (\cdot) \in I_g
    \end{split}
\end{gather}
\fi
We describe here in full detail the implementation of RIA using the RICE regularizer and how RICE still fits the causal model we define in Section \ref{sec: causalmodel}. %We do not go over VREx or IRM since our causal model was derived from theirs.

Let the the support of a distribution be the subset of its domain where it has nonzero measure. This is denoted $\textsf{supp}(P):= \{x \in \textsf{dom}(P) \mid P(x) >0 \}$.

We need to define a \emph{single training environment} in order to use RICE. We can do this by taking a \emph{mixture} of training environments:
\begin{equation}
    P^{tr}:= \sum\limits_{e \in \mathcal{E}_{tr}: \sum_{e \in \mathcal{E}_{tr}} \lambda_e=1, \lambda_e\geq 0} \lambda_e \cdot P^e
\end{equation} 
We will use $P^{tr}$ to sample the training datasets $D_{tr}:=\sqcup_{e\in \mathcal{E}_{tr}}D^e$, $D^e \subset \textsf{supp}(P^e)$ for $e \in \mathcal{E}_{tr}$. 
%\end{definition}

RICE assumes a causal model. The causal model we define in Section \ref{sec: causalmodel} is compatible with the causal model of RICE. The causal model of RICE assumes that, given the data, the label is generated by the map $Y=m(c_*(X,A),\eta)$ where $\eta$ is an exogenous variable, $c_*$ coincides with the map we defined in Section \ref{sec: causalmodel} and $m$ is any label producing map. RICE is formulated as a constrained optimization problem:

\begin{subequations}
\begin{equation}
\min_{\theta} \mathbb{E}_{(G,Y) \sim P^{tr}}[l(h_{\theta}(G),Y)]
\end{equation}
\begin{equation}
\text{s.t. }h_{\theta} \circ T= h_{\theta}, \forall T \in \mathcal{I}_{c_*}(\textsf{supp}(P^{tr}))
\end{equation}
\end{subequations}
where %$supp(P)= \{x | P(x) \neq 0\}$ and 
$\mathcal{I}_{c_*}(\textsf{supp}(P^{tr}))$ is defined below: 
\begin{definition}\label{def: I_g}(Causal Essential Invariant Transformations)
\cite{wang2022out}

Let $\mathcal{G}$ be the set of all tensor pairs $(X,A)$ from $\textsf{supp}(P^{tr})$ with the graph isomorphism relation.
\begin{align}
\begin{split}
\mathcal{I}_{c_*}(\mathcal{G})= \{T_i\mid c_*(X_1,A_1)=c_*(X_2,A_2) \Rightarrow\\
\exists T_1...T_k \text{ with } c_*\circ T_i= c_* \forall i, \text{ s.t. }\\
T_1 \circ \cdots \circ T_k(X_1,A_1)=(X_2,A_2)\\
\text{ and }   
\forall (X_1,A_1), (X_2,A_2) \in \mathcal{G}\}%e, e \in \mathcal{E}_{tr}\}
\end{split}
\end{align}
We notice that a subset of the causal essential invariant transformations are just the invertible data augmentations which satisfy $c_* \circ T = c_*$. Implementing these data augmentations, such as edge addition and deletion on graphs, to approximate $\mathcal{I}_{c_*}(\mathcal{G})$ is simple and effective for graphs. We can thus narrow down the number of hyperparameters.
\end{definition}
\begin{proposition}\label{prop: I_ginvertible-appendix}
The $\mathcal{I}_{c_*}(\mathcal{G})$ of Definition \ref{def: I_g} contains the set $\mathcal{I}^{\text{inv}}_{c_*}(\mathcal{G})$ of invertible transformations on data support $\mathcal{G}$ that satisfy ${c_*}\circ T={c_*}$.
\end{proposition}
\begin{proof}
We show that if $T$ is invertible and satisfies ${c_*} \circ T= {c_*}$, then $T \in \mathcal{I}_{c_*}(\mathcal{G})$.

\textbf{1. }We first show that the identities 
$\{I_{n_0}\}_{n_0\leq N_e}$, which depend on the number of graph nodes $n_0$, is in $\mathcal{I}_{c_*}(\mathcal{G})$.

Let $(X_1,A_1)\approx(X_2,A_2)$ be two representations of two isomorphic graphs on $n_0$ nodes, then we have that ${c_*}(X_1,A_1)={c_*}(X_2,A_2)$ and that $I_{n_0}(X_1,A_1)=(X_2,A_2)$ for $I_{n_0}$ the identity on $(X_1,A_1)$.

\textbf{2. }
For any $(X_1,A_1), (X_2,A_2) \in S$,  ${c_*}(X_1,A_1)={c_*}(X_2,A_2)$ then there exists $T' \in \mathcal{I}_{c_*}(P)$ s.t. 
\begin{equation}
    I_{n_0} \circ T'(X_1,A_1)= T^{-1} \circ T \circ T'(X_1,A_1)= (X_2,A_2)
\end{equation}
This shows that both $T$ and $T^{-1}$ are in $\mathcal{I}_g(\mathcal{G})$ for all $T$ invertible over all graph sizes in the data support $\mathcal{G}$.

%The converse also holds. If $T \in \mathcal{I}_g(P)$ then it is invertible since if $g(X_1,A_1)=g(X_2,A_2)$ then there exists $T, T' \in \mathcal{I}_g(P)$ s.t. $ T(X_1,A_1)= (X_2,A_2)$ and $T'(X_2,A_2)=(X_1,A_1)$, which shows that $T'= T^{-1}$ by substitution. Thus every $T \in \mathcal{I}_g(P)$ has an inverse $T'$.
\end{proof}
\iffalse
The following lemma from \cite{wang2022out} will also be helpful to us:
\begin{lemma}\label{lemma: rice-lemma2}
\cite{wang2022out}
If for some function $h$, $\forall T \in \mathcal{I}_{c_*}(\mathcal{G}), h\circ T= h$ on $S$ where $S$ is some data support, then there exists a function $\phi$ s.t. $h= \phi\circ c_*$.
\end{lemma}
\fi

Proposition \ref{prop: I_ginvertible-appendix}, tells us that we may use the invertible transformations on graphs such as edge deletion/addition in the regularization term of RICE. This means we can implement a RICE regularizer for the OoD loss by the following OoD regularization term:

\begin{equation}
    \textbf{OoD-Reg}_{RICE}(h_{\theta},\textsf{supp}(P^{tr})) := 
 \frac{1}{\lvert \mathcal{E}_{tr}\rvert}\sum_{e=1}^{\lvert \mathcal{E}_{tr}\rvert }\mathbb{E} [\max_{T \in \mathcal{I}_{\text{edge}}^{\text{inv}}(\mathcal{G})} \lvert ( h_{\theta} \circ T(\mathbf{G}_{w}^e)- h_{\theta}(\mathbf{G}_{w}^e)\rvert_2]
\end{equation}
where $\mathbf{Y}^e$ is the label for environment $e \in \mathcal{E}_{tr}$, $\mathbf{G}_w^e$ are the adversarially augmented graphs for environment $e$ and $h_{\theta}$ is a graph representation learner.

\section{Hyperparameters and Dataset Information}\label{sec: hyperparamsdatasets}
\begin{table}[th]
\centering
%\resizebox{\columnwidth}{!}{
\begin{adjustbox}{width=.65\textwidth,center}
\begin{tabular}{cccccccc}
\hline \multicolumn{7}{c}{\scriptsize {Hyperparameters}}        \\ \hline
\scriptsize acc &
  \multicolumn{1}{c}{\scriptsize {\sc CMNIST} } &
  \multicolumn{1}{c}{\scriptsize {\sc SST2}} &
  \multicolumn{2}{c}{\scriptsize {\sc Motif}} &
  \multicolumn{2}{c}{\scriptsize {\sc AMotif}} &
  \multicolumn{1}{c}{\scriptsize {\sc Synth} } \\
  \cmidrule(lr){2-2}
          \cmidrule(lr){3-3}
          \cmidrule(lr){4-5}
          \cmidrule(lr){6-7}
          \cmidrule(lr){8-8}
\scriptsize covariate &
  \multicolumn{1}{c}{\scriptsize color} &
  \multicolumn{1}{c}{\scriptsize length} &
  \multicolumn{1}{c}{\scriptsize basis} &
  \multicolumn{1}{c}{\scriptsize size} &
  \multicolumn{1}{c}{\scriptsize basis} &
  \multicolumn{1}{c}{\scriptsize size} &
  \multicolumn{1}{c}{\scriptsize basis} \\
  \hline
\scriptsize lr               & \tiny 1e-3 & \tiny 1e-3 & \tiny 1e-3 & \tiny 1e-3 & \tiny 1e-3 & \tiny 1e-3 & \tiny 1e-3 \\
\scriptsize $lr_{adv}$         & \tiny 1e-4 & \tiny 1e-4 & \tiny 1e-4 & \tiny 1e-4 & \tiny 1e-4 & \tiny 1e-4 & \tiny 1e-4 \\
\scriptsize epochs           & \tiny 500  & \tiny 200  & \tiny 200  & \tiny 200  & \tiny 200  & \tiny 200  & \tiny 100  \\
\scriptsize num. edge augs.  & \tiny 10   & \tiny 10   & \tiny 10   & \tiny 10   & \tiny 10   & \tiny 10   & \tiny 10   \\
\scriptsize $k$              & \tiny 1    & \tiny 1    & \tiny 0    & \tiny 0    & \tiny 5    & \tiny 5    & \tiny 20   \\
\scriptsize arch             & \tiny GIN  & \tiny GIN  & \tiny GIN  & \tiny GIN  & \tiny GIN  & \tiny GIN  & \tiny GIN  \\
\scriptsize num layers       & \tiny 5    & \tiny 5    & \tiny 3    & \tiny 3    & \tiny 3    & \tiny 3    & \tiny 2    \\
\scriptsize $p_{edge}^{add}$ & \tiny 0.1  & \tiny 0.1  & \tiny 0.01 & \tiny 0.01 & \tiny 0.01 & \tiny 0.01 & \tiny 0.01 \\
\scriptsize $p_{edge}^{del}$ & \tiny 0.1  & \tiny 0.1  & \tiny 0.01 & \tiny 0.01 & \tiny 0.01 & \tiny 0.01 & \tiny 0.01
\end{tabular}
\end{adjustbox}
\caption{Superset of all hyper parameters shared across all datasets and shifts for all experiments.} 
\label{tab: hyperparameters}
\end{table}
We describe here some more information about each dataset we use in our experiments:

\begin{itemize}
    \item {\sc CMNIST} ~\cite{arjovsky2019invariant} Dataset is derived from the MNIST dataset from computer vision. It is curated by \cite{gui2022good}. Digits are colored according to their domains. Specifically, in covariate shift split, we color digits with $7$ different colors, and digits with the first $5$ colors, the $6$th color, and the $7$th color are categorized into training, validation, and test sets. 
    \item {\sc SST2}~\cite{socher2013recursive} Derived from a natural language sentiment classification dataset. Each sentence is transformed into a grammar tree graph, where each node represents a word with corresponding word embeddings as node features. The dataset forms a binary classification task to predict the sentiment polarity of a sentence. We select sentence lengths as domains since the length of a sentence should not affect the sentimental polarity.
    \item {\sc Motif}~\cite{wu2022discovering} Each graph in the dataset is generated by connecting a base graph and a motif, and the label is determined by the motif solely. Instead of combining the base-label spurious correlations and size covariate shift together as in~\cite{wu2022discovering}, the size and basis shifts are separated. Specifically, we generate graphs using five label irrelevant base graphs (wheel, tree, ladder, star, and path) and three label determining motifs (house, cycle, and crane). To create covariate splits, we select the base graph type and the size as domain features. There are no node attributes in this dataset.
    \item {\sc AMotif} (a modification of {\sc Motif} to have attributes) Taking the same graph structures from {\sc Motif}, we use node attributes of dimension $256$ all sampled from a $N(0,(e+1)^2)$, where $e$ is the environment index. Covariate shifts are achieved by changing the basis or size as in {\sc Motif} each shift indexed by some $e$.
    \item {\sc Synth} We construct a synthetic dataset as described in Section \ref{sec: experiments}. The dataset is a modification of {\sc Motif}, which generates data by a joining operation between causal and spurious graphs. In our construction, we construct $(X_C,A),(X_S,A)$ as in {\sc AMotif}. We let the joining operation be the map $(J_X(X_C,X_S),J_A(X_C,X_S))= c_{\xi}^{-1}(X_C+X_S,A)=(X,A)$ where $\xi$ are neural weights. We assume that the map $c_{\xi}$ is invertible and has an inverse $c_{\xi}^{-1}$ defined by a GIN neural network that maps from the graph $(X_C+X_S,A)$ to the graph $G=(X,A)$. GIN is not guaranteed to be injective, however it is a good enough approximation to one in practice. % additively by vector $\delta_{\xi} \sim N(0,\sigma)$. $N(0,\sigma(MLP(\tilde{e})))$, where $\tilde{e}$ is the one-hot encoding of the environment index and $\sigma \circ MLP $ is a fixed neural mapping to a number in $(0,1)$. 
    The label is defined by $Y=m(X_C,A)$ where $m$ is a MLP.  %and $\eta \sim  N(0,\sigma(MLP(\tilde{e})))$ where $\tilde{e}$ is a one-hot encoding of the environment index and $\sigma \circ MLP$ is a fixed neural mapping to a tensor of numbers in $(0,1)$. 
    We can further assume that $c_{*}$, the causal map, can be obtained by $c_*(X,A)= c_{\xi}(X,A)-s_{\xi}(X,A)$ where $c_*$ is deterministic and $\xi$ is initialized by $\xi \sim N(0,MLP(\tilde{e}))$. For the RIA-RICE implementation $c_*$ is assumed to exist and allows us to obtain a solution of the form $\phi\circ c_*$. For RIA-IRM and RIA-VREx, so long as our data generation process coincides with the model of \cite{arjovsky2019invariant} is satisfied, %Notice we assume here that for any $\delta$, there exists a data augmentation learnable by some neural network corresponding to the $\delta$. Theoretically this would require $c_*$ injective to use a universal approximation theorem. 
    The distribution shifts are induced by varying $\tilde{e}$ and thus affecting $\eta$ and $\alpha$ simultaneously. There are $4$ environments in $\mathcal{E}$. Two environments are combined together for training, the third  for validation, and the remaining environments are for testing.
\end{itemize}

We list in Appendix-Table \ref{tab: hyperparameters} the hyperparameters of our approachs on all datasets experimented with.  
\end{document}